\documentclass[]{article}
\usepackage[a4paper, total={6in, 8in}]{geometry}
\usepackage[title]{appendix}

\usepackage{microtype}
\usepackage{tikz}
\usepackage{pgfplots}
\usepackage{graphicx}
\usepackage{booktabs} 
\usepackage{hyperref}
\usepackage{multirow}
\usepackage[bottom]{footmisc}
\usepackage[square,numbers]{natbib}
\usepackage{amsmath}
\usepackage{amssymb}
\usepackage{mathtools}
\usepackage{amsthm}
\usepackage{bm}

\RequirePackage{algorithm}
\RequirePackage{algorithmic}

\usepackage{soul}

\theoremstyle{plain}
\newtheorem*{theorem*}{Theorem}
\newtheorem{theorem}{Theorem}

\newtheorem*{lemma*}{Lemma}
\newtheorem*{corollary*}{Corollary}
\newtheorem*{definition*}{Definition}

\theoremstyle{remark}

\DeclareMathOperator{\ilr}{ilr}
\DeclareMathOperator{\tr}{tr}

\pgfplotsset{compat=1.18}

\newcommand{\BibTeX}{B\kern-.05em{\sc i\kern-.025em b}\kern-.08em\TeX}

\begin{document}

\title{Explaining a probabilistic prediction on the simplex\\with Shapley compositions}
\author{Paul-Gauthier Noé\footnote{Laboratoire Informatique d'Avignon, Avignon Université, France}, Miquel Perelló-Nieto\footnote{University of Bristol, United Kingdom}, Jean-François Bonastre\footnotemark[1], Peter Flach\footnotemark[2]}
\date{}
\maketitle

\begin{abstract}
Originating in game theory, Shapley values are widely used for explaining a machine learning model’s prediction by quantifying the contribution of each feature’s value to the prediction. This requires a scalar prediction as in binary classification, whereas a multiclass probabilistic prediction is a discrete probability distribution, living on a multidimensional simplex. In such a multiclass setting the Shapley values are typically computed separately on each class in a one-vs-rest manner, ignoring the compositional nature of the output distribution. In this paper, we introduce \emph{Shapley compositions} as a well-founded way to properly explain a multiclass probabilistic prediction, using the Aitchison geometry from compositional data analysis. We prove that the Shapley composition is the unique quantity satisfying linearity, symmetry and efficiency on the Aitchison simplex, extending the corresponding axiomatic properties of the standard Shapley value. We demonstrate this proper multiclass treatment in a range of scenarios.
\end{abstract}

\paragraph{Remark} This work is published in the proceedings of ECAI 2024. The present document gathers the full paper with the supplementary material. For citing this work, please refer to the version in the ECAI's proceeding.

\newpage
\section{Introduction}

Many machine learning approaches
are regarded as black-boxes, making them unreliable for real-life applications where the model's predictions need to be understood or explained. In recent years, the interest in more interpretable models and explainability methods has therefore increased in the machine learning literature \cite{angelov2021explainable,gunning_xaiexplainable_2019}. 
One group of approaches, known as \emph{local explanation}, aims to measure the contribution of each input feature's value to the computation of the model's output.
Shapley values are widely used for this purpose \cite{vstrumbelj2014explaining,datta2016}, especially since the release of the SHAP toolkit \cite{NIPS2017_7062}\footnote{\url{https://shap.readthedocs.io/en/latest}}. 

Shapley values were introduced in cooperative game theory where a group of players work together to maximise a payoff. A set of Shapley values distributes the payoff over all the players according to their individual contribution to the total. The Shapley value is the unique quantity that satisfies a set of desired axiomatic properties \cite{shapley1953value}. For explaining a machine learning model's prediction, features are treated as players and the scalar output of the model as the total payoff.

The Shapley value is designed for a one-dimensional function's codomain. In game theory, the characteristic function takes a coalition of players and gives a payoff. In machine learning, for a given instance, the characteristic function takes a group of features and gives a scalar output measuring how the prediction changes when the values of the features are considered. For a two-class probabilistic classification, the prediction is essentially a scalar since the probabilities for the two classes sum to one. Therefore, the Shapley value framework can simply be applied to the logit transform of one of the probabilities\footnote{The logit maps the domain $]0,1[$ to the additive real line $\mathbb{R}$.}.

For more than two possible classes the output of the model is a discrete probability distribution or the output of a softmax function as commonly used by neural networks. Hence the output lives on a $(D-1)$-dimensional simplex, where $D$ is the number of classes. In this case, the Shapley value framework cannot be directly applied. Of course, one can compute Shapley values on each output probability separately, but this ignores the structure of the simplex where the relative values between the probabilities is what really matters, rather than the absolute value of a single probability.
\begin{figure}[!b]
    \centering
    \includegraphics[width=0.55\linewidth]{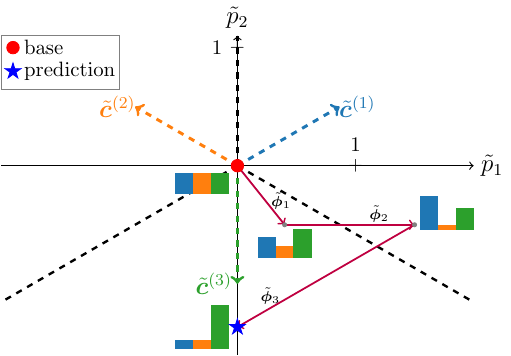}
    \caption{A synthetic example of a Shapley composition-based explanation. This shows a $2$-dimensional space isomorphic to the $3$-class simplex where each point is a probability distribution as visualised with the histograms. The dashed black rays separate the maximum probability regions for each class and the dashed coloured vectors show the direction in favour of one class and against the other two. The space is additive such that the features' contributions $\{\tilde{\bm{\phi}}_i \}_{\{1,2,3\}}$ translate the base distribution to the prediction.}
    \label{fig:intro}
\end{figure}

This paper presents \emph{Shapley composition} as an extension of the Shapley value to the space of discrete probability distributions, using the \emph{Aitchison geometry of the simplex} from the field of compositional data analysis. 
Compositional data \cite{aitchison1982,pawlowskymodeling} are vectors -- known as compositions -- living on a simplex (not necessarily a probability simplex). Compositional data analysis has been applied to geological and chemical data, for example, but also to discrete probability distributions \cite{egozcue2011evidence,egozcue2018evidence,noe2023representing}.
In the present paper, the probability distributions given by a classifier will be treated as compositional data in order to extend the Shapley value to multiclass classification.

Figure \ref{fig:intro} shows a synthetic example to provide some intuitions.
It shows a $2$-dimensional space isomorphic to the $3$-class-simplex where each point is a probability distribution also visualised as a histogram. The maximum probability regions, for each class, are clearly visible and separated by dashed black rays. 
Importantly, this space of probability distribution is additive thanks to the Aitchison geometry. 
The vectors show how the contribution of each feature changes -- in an additive manner -- the probability distribution (the ordering of the features can be chosen freely but the final point is fixed). 

In the example, the base distribution (the average prediction over all data\footnote{Note that this does not need to be the uniform distribution, i.e., the origin.}) is modified by the contribution $\tilde{\bm{\phi}}_1$ of the first feature. This goes mostly against class 2 such that the resulting distribution has the lowest probability for class 2. The angle between $\tilde{\bm{\phi}}_1$ and the class-3 direction being the lowest among the classes, the resulting distribution has the highest probability for this class. The second feature moves the distribution into the class-1 region, perpendicular to the class-3 direction. The probability for class 1 is now maximum by reducing the probability for class 2, keeping the relative weight for class 3 unchanged. The third feature moves the distribution away from class 1. The resulting distribution being on the class-3 direction, the probability is maximum for class 3 and uniform for the other two.

We fully formalise the approach in this paper, making the following contributions: 
\begin{itemize}
    \item We define \emph{Shapley composition} as a principled multidimensional extension of the Shapley value to the probability simplex,
    \item We prove that the Shapley composition is the unique quantity satisfying the set of desired properties known as \emph{linearity}, \emph{symmetry} and \emph{efficiency} on the simplex equipped with the Aitchison geometry,
    \item We demonstrate the advantages of Shapley compositions for explaining a multiclass probabilistic prediction in machine learning.
\end{itemize}
The paper is structured as follows. 
Section \ref{sec:related} briefly reviews related work. 
Section \ref{sec:shapley} recalls the standard definition of the Shapley value and its use in binary classification. 
Section \ref{sec:compo} presents the necessary tools from compositional data analysis: in particular, the Aitchison geometry of the simplex and the isometric log-ratio transformation. 
Section \ref{sec:shapcompo} defines the Shapley composition as an extension of the Shapley value framework to the multidimensional simplex using the Aitchison geometry. 
Section \ref{sec:explain} shows with intuitive examples and visualisations how Shapley compositions can be used for explaining multiclass probabilistic predictions. 
Section \ref{sec:conclud} provides a short discussion and concludes the paper\footnote{The code and Jupyter notebooks are available on the github page: \url{https://github.com/orgs/shapley-composition}.}.
\newpage
\section{Related work}\label{sec:related}

There is a plethora of methods in the literature to explain and better understand predictive models.
They focus on different aspects of the task, from possible dataset biases, the feature importance with respect to the target, the parameters of a model after training, or the model predictions \cite{sokol_explainability_2020}.
Some methods explain the influence of the features on the model's performance: e.g., Permutation Feature Importance for random forest \cite{breiman_random_2001}, which was later extended to the model agnostic Model Reliance \cite{fisher_all_2019}.
Other methods focus on how individual features influence the model's predictions: e.g., Local Interpretable Model-agnostic Explanations (LIME) \cite{ribeiro_why_2016}, Individual Conditional Explanation \cite{goldstein_peeking_2015}, Partial Dependence-based Feature Importance \cite{greenwell_simple_2018}, Marginal Effect \cite{apley_visualizing_2020},  Accumulated Local Effect \cite{apley_visualizing_2020}, and Shapley value-based approaches \cite{vstrumbelj2014explaining,datta2016,NIPS2017_7062}.

We base our work on the Shapley value framework as one of the most well-founded feature influence methods. Inherently two-class, it has been applied to multiclass problems by explaining the influence of the features in a one-vs-one or one-vs-rest manner \cite{yoo2020,lamens_explaining_2023}, hence losing information that can be obtained by properly considering the full distribution. \citet{utkin_imprecise_2021,Utkin2023} explicitly consider the classifier output as a probability distribution, and measure the change in prediction in terms of statistical distance or divergence rather than in terms of difference between scalar predictions. However, even if this approach can measure the strength of a feature's value effect, it loses its directional information.

A recent work presented by \citet{franceschi2023} (later extended \citep{franceschi2024explaining}) introduces stochastic characteristic functions to deal with models that output a random variable. With a categorical random variable, their approach can be used for explaining a multiclass classifier by allowing probabilistic statements about the likelihood of a feature to flip the decision from one class to another. In contrast, the approach we propose does not require an additional stochastic process but does not permit such a probabilistic statement. Instead, our approach is geometrical, by measuring how a feature moves the prediction on the probability simplex. In this way, it constitutes a natural extension of the standard Shapley value to the simplex for multiclass applications.

\section{The Shapley value in machine learning}
\label{sec:shapley}

This section briefly recalls Shapley values as used for explaining features' contribution on a scalar prediction in machine learning.
Let $f:\mathcal{X}\to\mathbb{R}$ be a learned model one wants to \emph{locally} explain where $f(\bm{x})$ is the prediction on the instance $\bm{x}\in\mathcal{X}\subset\mathbb{R}^d$. 
Let $\text{Pr}$ be the probability distribution of the data over $\mathcal{X}$ (usually unknown but approximated by empirical averages). 
Let $S\subseteq \mathcal{I}=\{1,2,\dots d\}$ be a subset of indices where $d$ is the number of features. $\bm{x}_S$ refers to an instance $\bm{x}$ restricted to the features with indices in $S$.

When an instance $\bm{x}$ is observed, the expected value of the prediction is simply $\mathbb{E}[f(\bm{X}) \mid \bm{x}] = f(\bm{x})$. However, when only $\bm{x}_S$ is given with $\mathcal{S} \neq \mathcal{I}$, there is uncertainty about the non-observed features and the expected prediction given $\bm{x}_S$ is computed as $\mathbb{E}_{\text{Pr}}[f(\bm{X}) \mid \bm{x}_S] = \int_{\bm{x} \in \mathcal{X}}f(\bm{x})\text{Pr}(\bm{x} \mid \bm{x}_S)d\bm{x}$. The change in prediction when the values of the features indexed by $S$ are observed is measured by the characteristic function:
\begin{equation}
  \label{eq:valuefunction}
  \begin{aligned}
    v_{f,\bm{x},\text{Pr}}: 2^{\mathcal{I}} &\to \mathbb{R},\\
    S &\mapsto \mathbb{E}_\text{Pr}[f(\bm{X})\mid \bm{x}_S] - \mathbb{E}_\text{Pr}[f(\bm{X})],
  \end{aligned}
\end{equation}
where $2^{\mathcal{I}}$ is the set of all subsets of $\mathcal{I}$.
The contribution of the feature indexed by $i \notin S$ to the prediction, given the known values for the features indexed by $S$, is given by:
\begin{equation}
  \label{eq:contrib}
  c_{f,\bm{x},\text{Pr}}(i,S) = v_{f,\bm{x},\text{Pr}}({S\cup\{i\}}) - v_{f,\bm{x},\text{Pr}}(S).
\end{equation} 
The total contribution of the $i$th feature is computed by averaging this quantity over all possible coalitions $S$ as follows:
\begin{equation}
  \phi_{i}\left( {f,\bm{x},\text{Pr}} \right) = \frac{1}{d!} \sum_{\pi}c_{f,\bm{x},\text{Pr}}(i,\pi^{<i}),
\end{equation}
where $\pi$ is a permutation of the set $\mathcal{I}$ of indexes and $\pi^{<i}$ is the set of indexes before $i$ in the ordering given by $\pi$. For better clarity, ``${f,\bm{x},\text{Pr}}$'' or simply ``$\bm{x},\text{Pr}$'' will be dropped from the equations.

This quantity is known as the Shapley value for the $i$th feature. It originates from cooperative game theory and is the unique quantity respecting a set of desired axiomatic properties \cite{shapley1953value,vstrumbelj2014explaining}: 
\begin{description}
\item[~~~Linearity] with respect to the model: \\$\alpha, \beta \in \mathbb{R}$, $\forall i \in \mathcal{I},~\phi_{i}\left( {\alpha f +\beta g }\right) = \alpha \phi_i \left( f \right) + \beta \phi_i\left(g\right)$;
\item[~~~Symmetry:]~\\$\forall S \subseteq \mathcal{I} \backslash \{i,j\}$, $v\left({S \cup \{ i \} }\right) = v\left({S \cup \{ j \} }\right) \Rightarrow \phi_i = \phi_j$;
\item[~~~Efficiency:] The ``centered'' prediction is additively separable with respect to the Shapley values:
\begin{equation}
  f(\bm{x})-\mathbb{E}_{\text{Pr}}[f(\bm{X})] = \sum_{i=1}^{d} \phi_i\left( {f,\bm{x},\text{Pr}} \right).
\end{equation}
\end{description}
Efficiency ensures that the change in prediction when the features are observed is distributed among them. In other words, the cumulative sum of the Shapley values moves the averaged prediction (also called \emph{base} prediction) to the actual one.

The Shapley value is designed for a characteristic function with a scalar codomain. For explaining machine learning models which output multidimensional discrete probability distributions, like in multiclass classification, one could explain each output dimension separately, 
resulting in a one-vs-rest comparison. However, this approach ignores the relative information between each probability and ignores the compositional nature of the discrete probability distributions. Indeed, the probabilistic output of a classifier lives on a multidimensional simplex. The latter is the sample space of \emph{compositional data} briefly reviewed in the next section.

\section{Compositional data}
\label{sec:compo}

Compositional data carries relative information. Each element of a composition \emph{describes a part of some whole} \cite{pawlowskymodeling}, such as vectors of proportions, concentrations, and discrete probability distributions. An $D$-part composition is a vector of $D$ non-zero positive real numbers that sum to a constant $k$. Each element of the vector is a part of the \emph{whole} $k$. The sample space of compositional data is the $(D-1)$-dimensional simplex: 
\begin{equation}
\mathcal{S}^D = \left\{ \bm{x} = [x_1, x_2,\dots x_{D}]^T \in \mathbb{R}^{*D}_{+} \mid \sum_{i=1}^{D} x_i = k \right\}.
\end{equation}
In a composition, only the relative information between parts matters and John Aitchison introduced the use of log-ratios of parts to handle this \cite{aitchison1982}. He defined several operations on the simplex which leads to what is called the \emph{Aitchison geometry of the simplex}.

\subsection{The Aitchison geometry of the simplex}
Only the relative information between the parts of a composition matters. Compositions are therefore scale-invariant. This is materialised by the closure operator defined for $k>0$ as: 
\begin{equation}
\mathcal{C}\left(\bm{x} \right) = \left[ \frac{k x_1}{\lVert \bm{x} \rVert_1}, \frac{k x_2}{\lVert \bm{x} \rVert_1} ,\dots \frac{k x_{D}}{\lVert \bm{x} \rVert_1} \right]^T \in \mathcal{S}^D,
\end{equation}
where $\bm{x} \in \mathbb{R}_+^{*D}$ and $ \lVert \bm{x} \rVert_1 = \sum_{i=1}^{D} \lvert x_i \rvert$.

Aitchison defined on the simplex the following three operations \cite{aitchison2001}:
\begin{description}
\item[Perturbation:] $\bm{x}\oplus \bm{y} = \mathcal{C}\left([x_1y_1,\dots x_{D}y_{D}]\right)$,
  seen as an addition between two compositions;
\item[Powering:] $\alpha \odot \bm{x} = \mathcal{C}\left([x_{1}^{\alpha},\dots x_{D}^{\alpha}]\right)$,
  seen as a multiplication by a scalar $\alpha \in \mathbb{R}$;
\item[Inner product:]
  \begin{equation}
    \langle \bm{x},\bm{y} \rangle_{\mathcal{A}} = \frac{1}{2D}\sum_{i=1}^{D} \sum_{j=1}^{D} \log \frac{x_i}{x_j}\log \frac{y_i}{y_j}.
    \label{eq:inner}
  \end{equation}
\end{description}
In this paper, since we are interested in classification problems where the set of classes represents a set of exhaustive and mutually exclusive hypotheses, the output of a probabilistic classifier is a discrete probability distribution over the set of classes. We therefore restrict ourselves to the \emph{probability simplex} where $k=1$.

\subsection{The isometric log-ratio transformation}
\label{sec:ilr}

A $(D-1)$-dimensional orthonormal basis of the simplex, referred to as an \emph{Aitchison} orthonormal basis, can be built. The projection of a composition into this basis defines an isometric isomorphism between $\mathcal{S}^D$ and $\mathbb{R}^{D-1}$. This is known as an isometric log-ratio (ILR) transformation \cite{egozcue2003isometric} and allows to express a composition into a Cartesian coordinate system preserving the metric of the Aitchison geometry. Within this real space, the perturbation, the powering and the Aitchison inner product are respectively the standard addition between two vectors, the multiplication of a vector by a scalar, and the standard inner product.

Given a composition $\bm{p} = \left[ p_1,\dots p_{D} \right]^T \in \mathcal{S}^D$ we write its ILR transformation as $\tilde{\bm{p}} = \ilr \left( \bm{p} \right) = \left[ \tilde{p}_1,\dots \tilde{p}_{D-1} \right]^T \in \mathbb{R}^{D-1}$. The $i$th element $\tilde{p}_i$ of $\tilde{\bm{p}}$ is obtained as: $\tilde{p}_i = \langle \bm{p}, \bm{e}^{(i)} \rangle_{\mathcal{A}}$ where the set $\{\bm{e}^{(i)} \in \mathcal{S}^D\}_{1\leq i \leq D-1}$ forms an \emph{Aitchison} orthonormal basis of the simplex. The basis can be obtain through the Gram-Schmidt procedure or by building a sequential binary partition \cite{egozcue2003isometric,egozcue2005groups}. Examples are discussed in Section \ref{sec:balances}.

In the introductory example of Figure \ref{fig:intro}, the $2$-dimensional ILR space isomorphic to the $3$-class probability simplex was constructed as follows:
\begin{equation*}
        \tilde{p}_1 = \frac{1}{\sqrt{2}} \log \frac{p_1}{p_2}, \quad
        \tilde{p}_2 = 
                       \sqrt{\frac{2}{3}}\log \frac{\sqrt{p_1 p_2}}{p_3}. 
\end{equation*}
Hence, the $x$-axis compares the probabilities for classes 1 and 2,
the $y$-axis compares the probability for class 3 with the geometric mean of $p_1$ and $p_2$, and
the origin corresponds to the uniform distribution, i.e.,
the neutral element for the perturbation. 
Note that the perturbation can be seen as a Bayesian update: the perturbation of a prior by a likelihood function gives the posterior. In the space of isometric log-ratio transformed distributions, the Bayes update is a vector translation.

\section{Shapley composition on the simplex}
\label{sec:shapcompo}

In this section we will use the Aitchison geometry to extend the Shapley value from Section \ref{sec:shapley} to the simplex for explaining a multiclass probabilistic prediction.
Let $\bm{f}:\mathcal{X}\to\mathcal{S}^D$ be a learned model 
which outputs a prediction on the $(D-1)$-dimensional probability simplex $\mathcal{S}^{D}$. In order to properly consider the structure of the simplex and the relative information between the probabilities, the model's output is treated as compositional data using the operators from the Aitchison geometry of the simplex. We therefore rewrite the characteristic function and the contribution of Equations \ref{eq:valuefunction} and \ref{eq:contrib} as follows:
\begin{equation}
  \label{eq:valuefunctionsimplex}
  \begin{aligned}
    \bm{v}_{\bm{f},\bm{x},\text{Pr}}: 2^{\mathcal{I}} &\to \mathcal{S}^D,\\
    S &\mapsto \mathbb{E}^{\mathcal{A}}_\text{Pr}[\bm{f}(\bm{X})\mid \bm{x}_S]\ominus \mathbb{E}^{\mathcal{A}}_\text{Pr}[\bm{f}(\bm{X})].
  \end{aligned}
\end{equation}
\begin{equation}
  \bm{c}_{\bm{f},\bm{x},\text{Pr}}(i,S) = \bm{v}_{\bm{f},\bm{x},\text{Pr}}({S\cup\{i\}}) \ominus \bm{v}_{\bm{f},\bm{x},\text{Pr}}(S),
\end{equation}
where $\bm{a}\ominus\bm{b}$ is the perturbation $\bm{a} \oplus \left( (-1)\odot \bm{b}\right)$ which corresponds to a subtraction between two compositions, and where the $\mathcal{A}$ in superscript highlights the fact that the expectation is taken with respect to the Aitchison measure. 
This can be computed as: $\mathbb{E}^{\mathcal{A}}[\bm{Y}] = \ilr^{-1}\left( \mathbb{E} \left[ \ilr\left( \bm{Y} \right) \right] \right)$,
where $\mathbb{E}^{\mathcal{A}}$ refers to the expectation with respect to the Aitchison measure while $\mathbb{E}$ refers to the expectation with respect to the Lebesgue measure \cite{pawlowskymodeling}.

The Shapley quantity expressing the contribution of the $i$th feature's value on a prediction can  be expressed on the simplex as the composition $\bm{\phi}_{i}$ given by:
\begin{equation}
  \bm{\phi}_{i}\left( {\bm{f},\bm{x},\text{Pr}} \right) = \frac{1}{d!} \odot \underset{\pi}{\bigoplus}\bm{c}_{\bm{f},\bm{x},\text{Pr}}(i,\pi^{<i}).
\end{equation}
We call this quantity \emph{Shapley composition}. Note that the average is here with respect to the Aitchison geometry, i.e.~with perturbations and a powering rather than sums and a scaling.

The following is the main theoretical result of the paper. 
\begin{theorem}
\label{theo:shap}
The Shapley composition is the unique quantity that satisfies the following properties on the Aitchison simplex:
\begin{description}
    \item[~~~Linearity] with respect to the model: \\$\alpha, \beta \in \mathbb{R}$, $\forall i \in \mathcal{I},\\ \bm{\phi}_i\left( \alpha \odot \bm{f} \oplus \beta \odot \bm{g} \right) = \alpha \odot \bm{\phi}_i\left( \bm{f} \right) \oplus \beta \odot \bm{\phi}_i\left( \bm{g} \right)$;
    \item[~~~Symmetry:] $\\ \forall S \subseteq \mathcal{I} \backslash \{i,j\}$, $\bm{v}\left({S \cup \{ i \} }\right) = \bm{v}\left({S \cup \{ j \} }\right) \Rightarrow \bm{\phi}_i = \bm{\phi}_j$;
    \item[~~~Efficiency:]
    \begin{equation}
    \underset{i=1}{\overset{d}\bigoplus} \bm{\phi}_{i}\left( \bm{f},\bm{x},\text{Pr} \right) = \bm{f}(\bm{x}) \ominus \mathbb{E}^{\mathcal{A}}_{\text{Pr}}[\bm{f}(\bm{X})].
    \end{equation}
\end{description}
\end{theorem}
A proof of this result is given in Appendix \ref{app:proof}. Shapley compositions are thus the natural multidimensional extension of the Shapley value framework on the Aitchison simplex. 
In the next section
we give a number of compelling examples of how this can be used to explain multiclass probabilistic predictions. 

\section{Explaining a multiclass prediction with Shapley compositions}
\label{sec:explain}

Given a probabilistic prediction $\bm{f}(\bm{x}) \in \mathcal{S}^D$, the Shapley composition $\bm{\phi}_i \left( \bm{f},\bm{x},\text{Pr}\right)$ describes the contribution of the $i$th feature value to the prediction. The efficiency property shows how the probability distribution is perturbed from the \emph{base} distribution $\mathbb{E}^{\mathcal{A}}_{\text{Pr}}[\bm{f}(\bm{X})]$, i.e.~the expected prediction regardless of the current input, to the actual prediction $\bm{f}(\bm{x})$. In the standard Shapley formulation recalled in Section \ref{sec:shapley}, the prediction is one-dimensional such that the Shapley quantity is a scalar. In applications where there are more than two possible classes, the prediction is multidimensional such that the Shapley quantity (the Shapley composition) is too. Both live in the same space: the probability simplex. In this section, we discuss how the set of Shapley compositions can be analysed to better understand the contribution and influence of each feature's value on the prediction.

\subsection{Visualisation in an isometric-log-ratio space}

The Shapley compositions can be visualised in a $(D-1)$-dimensional Euclidean space isometric to the simplex with the ILR transformation presented in Section \ref{sec:ilr}. As we will see, this space is intuitive since it is a standard real vector space and it is additive. In what follows, we discuss some examples of Shapley composition-based explanations in an ILR space.

  \begin{figure}[!b]
    \centering
    \includegraphics[width=0.6\linewidth]{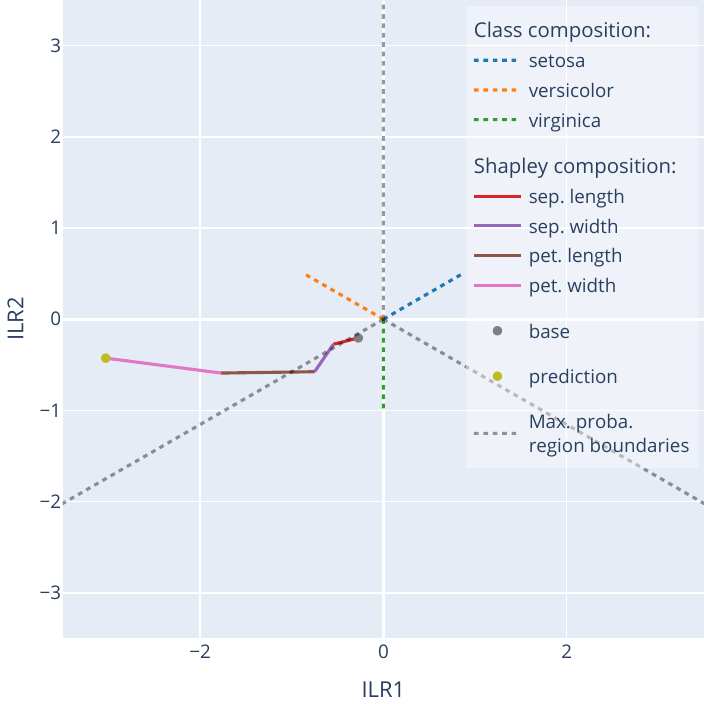}
    \caption{The sum of the Shapley compositions in an ILR space from the base distribution to the prediction for the classification of an Iris instance.}
    \label{fig:3classesshapsum}
    \end{figure}

%
%
\begin{figure}[!b]
    \centering
    \includegraphics[width=0.6\linewidth]{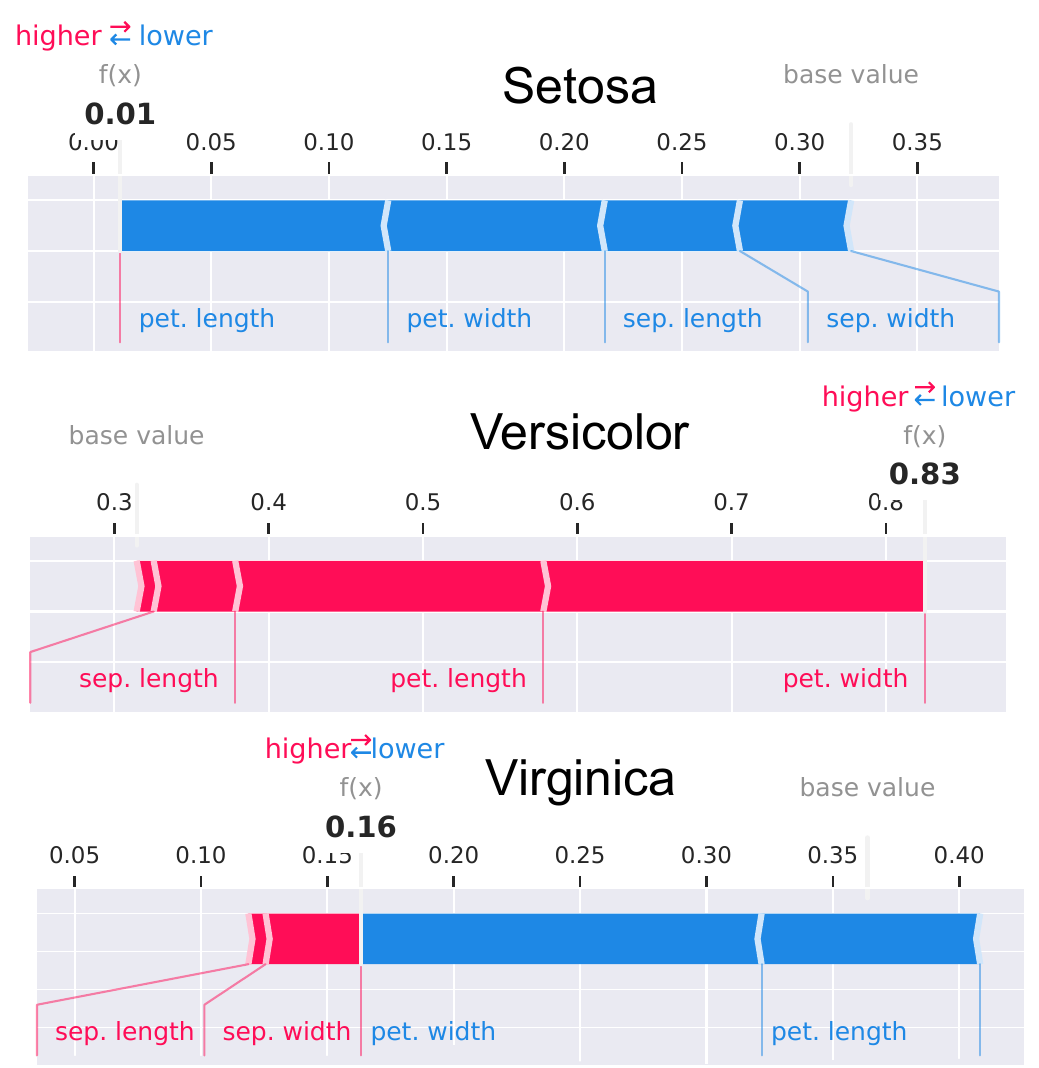}
     \caption{Visualisation of the Shapley values for each class in a one-vs-the-rest manner for the same instance as in Figure~\ref{fig:3classesshapsum}, obtained using the SHAP toolkit \cite{NIPS2017_7062}. The red/blue bars represent positive/negative contributions of each feature on the prediction.}
    \label{fig:3classes:ovr:shap}
\end{figure}

\subsubsection{Three classes}
\label{sec:3classes}

Our first illustration uses the well-known Iris classification dataset consisting of a set of flowers described by 4 features: sepal length and width, and petal length and width. The aim of the classification task is to predict to which of the three species (\emph{setosa}, \emph{versicolor} and \emph{virginica}) a flower belongs.

In the present example, a Support Vector Machine (SVM) with a radial basis function (rbf) kernel is used as a classifier. Pairwise coupling \cite{wu2003probability} is used to obtain a probabilistic prediction. Figure 
\ref{fig:3classesshapsum} shows the explanation of the classifier prediction for one \emph{versicolor} instance 
where the Shapley compositions 
move the distribution from the base to the prediction. Having the highest norm, the petal width and length are the features contributing the most to the prediction and move the base distribution into the \emph{versicolor} maximum probability region (maximum probability region boundaries are the dashed gray rays). Class-compositions are represented by coloured dashed vectors. A class-composition is defined as a unit norm composition going straight to the direction of one class and uniformly against all the others (see Appendix \ref{app:classcompo} for a formal definition).
Note that the class-compositions are not mutually orthogonal. This is because a positive contribution toward one class has necessarily to lead to a negative contribution toward at least another class to preserve the structure of the simplex.  

The Shapley composition for the petal length is almost orthogonal to the \emph{virginica} class-composition: for this instance, this feature does not contribute to the weight of the predicted probability for this class. Having a Shapley composition going straight to the opposite direction of one class-composition would suggest that the corresponding feature's value contributes to rejecting this class. This is somewhat the case for the sepal length. However, because its Shapley composition has a low norm, this feature contributes little to the prediction.

Alternatively, one could analyse this instance by applying the standard Shapley value in a one-vs-rest manner, explaining the feature contributions separately for each class. Figure \ref{fig:3classes:ovr:shap} shows how each explanation is usually visualised with the SHAP toolkit \cite{NIPS2017_7062}. The prediction is explained for each class one-by-one independently from one another, which makes it hard to appreciate the influence of one feature on the full distribution. Moreover, there is no guarantee that the intermediate full distribution remains on the simplex.
In contrast, with our approach, the influence of one feature's value on the full prediction can be analyse with a single quantity, the Shapley composition, in a single coherent and easily interpretable plot.

\subsubsection{Four classes}

In a four-class example, the simplex is $3$-dimensional. We illustrate this with a simple handwritten digit recognition task\footnote{We use the scikit-learn's digits dataset \cite{pedregosa2011scikit}.}. It consists of classifying an $8\times8$ image as representing one of the digits 0, 1, 2 or 3. Since there are 64 pixels, considering each pixel as a feature would correspond to 64 Shapley compositions. Moreover, the pixels will be highly correlated. Since our goal here is to provide simple illustrative examples, we reduce the number of features to 6 using a principal component analysis for better clarity and conciseness. An SVM with a rbf kernel and pairwise coupling is again used as a probabilistic classifier. A similar analysis as before can be applied here but within a $3$-dimensional plot as illustrated in Figure \ref{fig:4classesshapsum}. 

\begin{figure}[ht]
  \centering
  \includegraphics[width=0.6\linewidth]{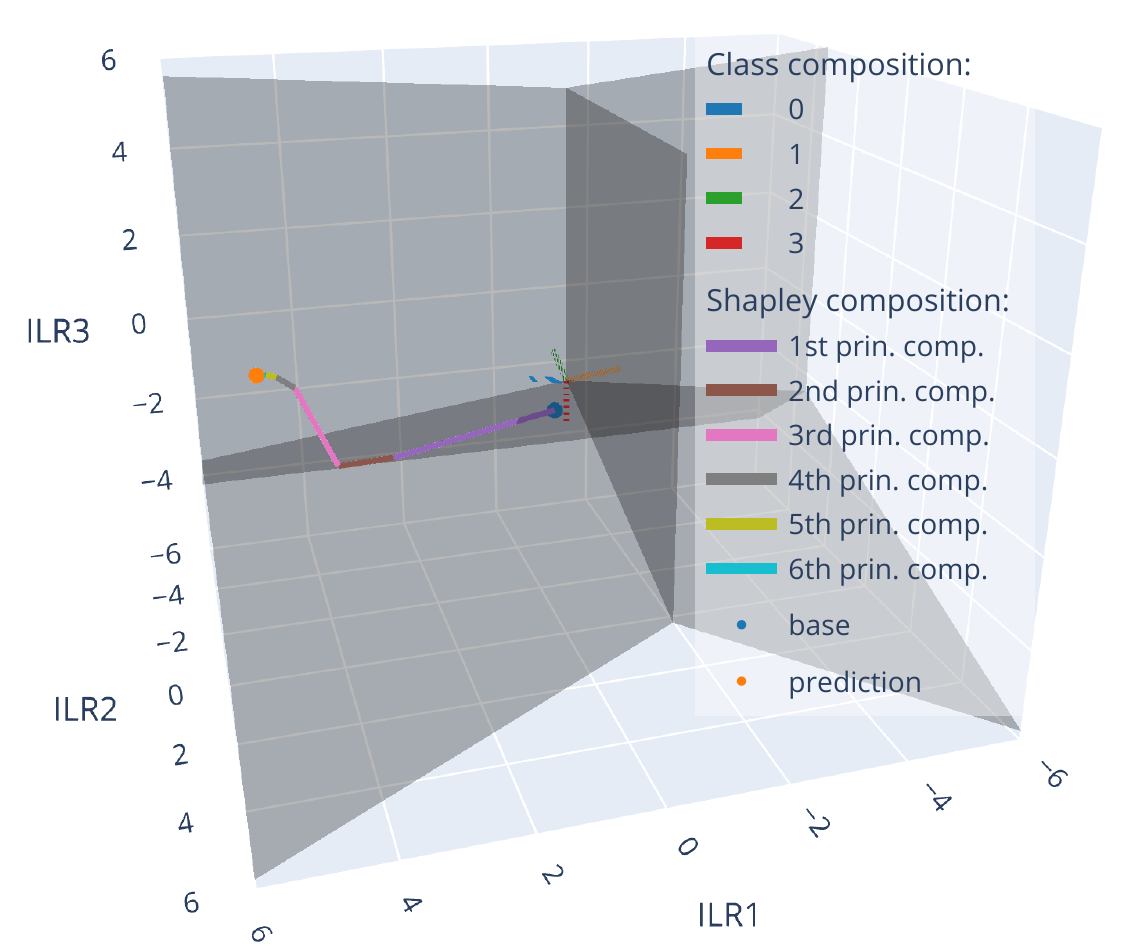}
  \caption{Shapley compositions in a $3$-dimensional ILR space for a four classes digit recognition task. The Shapley compositions are summed in the ILR space from the base distribution to the prediction. The gray transparent walls mark out the four maximum probability decision regions.}
  \label{fig:4classesshapsum}
\end{figure}

To better understand how this space is divided into four regions--each representing the maximum probability region for one class--one can think about the shape of a methane molecule. The hydrogens correspond to the vertices and the carbon to the center of a tetrahedron i.e.~a $3$-dimensional simplex. The relative positions of the class-compositions in the ILR space are the same as the bonds between the carbon and hydrogen: the angles are $\approx 109.5^{\circ}$. In this example, the tested instance is a 0\footnote{More examples and better visualisations can be obtained from the Jupyter notebooks: \url{https://github.com/orgs/shapley-composition}.}.

%

\subsection{More classes: groups of classes and balances}
\label{sec:balances}

When more than three classes are involved, the dimensions of the ILR space cannot be visualised all at once. However, $2$ or $3$-dimensional subspaces can still be visualised. In order to select the ILR components to investigate, one needs to understand what they refer to. In this section, we briefly discuss the interpretation of the ILR components.

A component of an ILR space can be interpreted as a \emph{balance}, i.e.~a log-ratio of two geometrical means of probabilities \cite{egozcue2003isometric,egozcue2005groups,pawlowskymodeling}: one giving the central values of the probabilities in one group of classes and one for another group of classes. Therefore, a balance is here comparing the weight of two groups of classes. The set of balances is built such that they are geometrically orthogonal meaning they provide non-redundant information\footnote{Not to be confused with statistical uncorrelation \cite{pawlowskymodeling}.}. This can be illustrated by a sequential binary partition or bifurcation tree. Two examples are given in Figures \ref{fig:bifurc1} and \ref{fig:bifurc2}. Figure \ref{fig:bifurc1} shows the bifurcation tree corresponding to the basis obtained with the Gram-Schmidt procedure as in \cite{egozcue2003isometric} which is the one used in the examples of Figures \ref{fig:3classesshapsum} and \ref{fig:4classesshapsum}  with respectively $D=3$ and $D=4$. Each node of the tree is a balance, i.e.,~an ILR component. The first balance $\tilde{p}_1$ compares the probability for class $1$ with the probability for class $2$. Each next balance then recursively compares the probability for the next class with the probabilities for the previous classes independently of all the others.

In some applications, one may be interested in particular comparisons of groups of classes not necessarily given by a basis in the form of Figure \ref{fig:bifurc1}. For instance, as in an example presented in \cite{egozcue2005groups}, if one wants to compare political parties or groups, it may be pertinent to have a balance comparing left and right-wing groups. But sometimes  there are no obvious relevant comparisons to study. In the handwritten digit recognition problem, one may want to compare odd with even numbers or primes with non-primes (although, being essentially a shape recognition problem, and the shape of the numbers being independent of their arithmetic properties, such comparisons may not be pertinent). 

  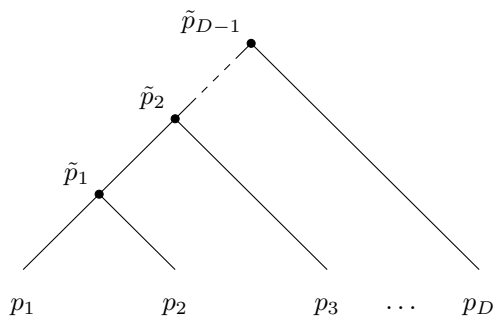
\begin{figure}[ht]
    \centering
    \begin{tikzpicture}[scale=2]
  \draw (0,0) -- (1.1,1.1);
  \draw[dashed] (1.1,1.1) -- (1.4,1.4);
  \draw (1.4,1.4) -- (1.5,1.5);
  \draw (0.5,0.5) -- (1,0);
  \draw (1,1) -- (2,0);
  \draw (1.5,1.5) -- (3,0);

  \filldraw[black] (0.5,0.5) circle (0.75pt) node[anchor=south east]{$\tilde{p}_{1}$};
  \filldraw[black] (1,1) circle (0.75pt) node[anchor=south east]{$\tilde{p}_{2}$};
  \filldraw[black] (1.5,1.5) circle (0.75pt) node[anchor=south east]{$\tilde{p}_{D-1}$};

  \draw[black] (0,-0.25) node{$p_1$};
  \draw[black] (1,-0.25) node{$p_2$};
  \draw[black] (2,-0.25) node{$p_3$};
  \draw[black] (2.5,-0.25) node{$\dots$};
  \draw[black] (3,-0.25) node{$p_{D}$};
\end{tikzpicture}

    \caption{Bifurcation tree corresponding to the basis obtained with the Gram-Schmidt procedure as in \cite{egozcue2003isometric} and used in the examples of Figures \ref{fig:3classesshapsum} and \ref{fig:4classesshapsum}.}
    \label{fig:bifurc1}
    \end{figure}
  \begin{figure}[ht]
      \centering
      \begin{tikzpicture}[scale=2]
  \draw (0,0) -- (2.25,2.25);
  \draw (2.25,2.25) -- (4.5,0);
  \draw (2,1.5) -- (0.5,0);
  \draw (1.25,0.25) -- (1,0);
  \draw (1,0.5) -- (1.5,0);
  \draw (3.25,0.25) -- (3,0);
  \draw (4.25,0.25) -- (4,0);
  \draw (1.75,1.75) -- (3.5,0);
  \draw (3,0.5) -- (2.5,0);
  \draw (2.75,0.75) -- (2,0);
  
  \filldraw[black] (1.75,1.75) circle (0.75pt) node[label={[anchor=south east]0.1mm:$\tilde{p}_{8}$}]{};
  \filldraw[black] (2.25,2.25) circle (0.75pt) node[label={[anchor=south east]0.1mm:$\tilde{p}_{9}$}]{};
  \filldraw[black] (3.25,0.25) circle (0.75pt) node[anchor=south west]{$\tilde{p}_{2}$};
  \filldraw[black] (4.25,0.25) circle (0.75pt) node[anchor=south west]{$\tilde{p}_{3}$};
  \filldraw[black] (1.25,0.25) circle (0.75pt) node[anchor=south west]{$\tilde{p}_{1}$};
  \filldraw[black] (1,0.5) circle (0.75pt) node[label={[anchor=south east]0.1mm:\small $\tilde{p}_{5}$}]{};
  
  \filldraw[black] (2,1.5) circle (0.75pt) node[anchor=south west]{$\tilde{p}_{7}$};
  \filldraw[black] (3,0.5) circle (0.75pt) node[anchor=south west]{$\tilde{p}_{4}$};
  \filldraw[black] (2.75,0.75) circle (0.75pt) node[anchor=south west]{$\tilde{p}_{6}$};
  
  \draw[black] (0,-0.25) node{$4$};
  \draw[black] (0.5,-0.25) node{$1$};
  \draw[black] (1,-0.25) node{$7$};
  \draw[black] (1.5,-0.25) node{$8$};
  \draw[black] (2,-0.25) node{$5$};
  \draw[black] (2.5,-0.25) node{$2$};
  \draw[black] (3,-0.25) node{$3$};
  \draw[black] (3.5,-0.25) node{$9$};
  \draw[black] (4,-0.25) node{$0$};
  \draw[black] (4.5,-0.25) node{$6$};
\end{tikzpicture}

      \caption{Bifurcation tree used in the $10$-class digit recognition task discussed in Section \ref{sec:balances} and in Figure \ref{fig:moreclasses35}.}
      \label{fig:bifurc2}
    \end{figure}
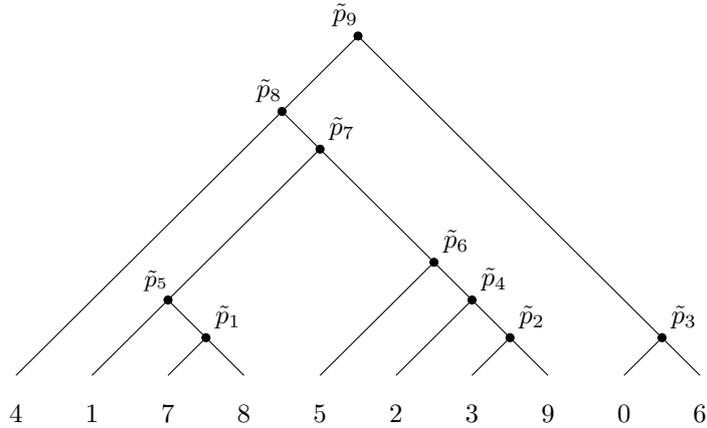

\begin{figure}[ht]
  \centering
  \includegraphics[width=0.7\linewidth]{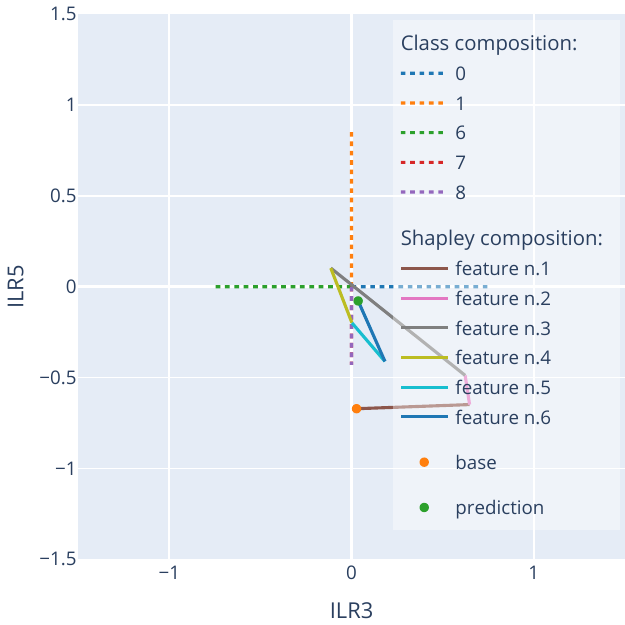}
  \caption{The sum of the Shapley compositions from the base to the prediction in the ILR subspace made of $\tilde{p}_3$ and $\tilde{p}_5$ for a test instance from class 2. $\tilde{p}_3$ compares the probability assigned for class 0 with the probability assigned for class 6 and $\tilde{p}_5$ compares the probability assigned for class 1 with the group of probabilities assigned for class 7 and 8. The color dashed vectors represent the class-compositions with non-zero projection.}
  \label{fig:moreclasses35}
\end{figure}

We use the basis of Figure \ref{fig:bifurc2} for a 10-class digit recognition task.
In this example, the bifurcation tree is obtained from the dendrogram of an agglomerative clustering of classes: for each class, the set of predictions is modelled by a logistic-normal \cite{aitchison1980}, with equal covariance, and classes are recursively merged with respect to the Mahalanobis distance.
Consider, in Figure \ref{fig:moreclasses35}, the third and fifth ILR dimensions ($\tilde{p}_3$ and $\tilde{p}_5$). 
Effectively, we are saying that we are interested in comparing the probability assigned for class 0 with the probability assigned for class 6, and in comparing the probability assigned for class 1 with the group of probabilities assigned for classes 7 and 8. $\tilde{p}_3$ depends only on the probability for the digits 0 and 6, and $\tilde{p}_5$ depends only on the probabilities for the digits 1, 7 and 8. The class-compositions for the other digits have a zero projection within this subspace and are therefore discarded in Figure \ref{fig:moreclasses35}. The class-compositions for 0 and 6 are orthogonal to the class-compositions for classes 1, 7 and 8. Indeed, the set of classes making the balance $\tilde{p}_3$ and the set of classes making $\tilde{p}_5$ have no intersection. 

In contrast, in the example of Figure \ref{fig:3classesshapsum}, $\tilde{p}_1$ is comparing the probabilities for the class \emph{setosa} with the probability for the class \emph{versicolor} and $\tilde{p}_2$ is comparing the probabilities for the class \emph{virginica} with the group of probabilities for \emph{setosa} and \emph{versicolor}. In Figure \ref{fig:3classesshapsum}, the class-compositions are exhaustively present and are therefore geometrically dependent and none of them are orthogonal. 
In Figure \ref{fig:moreclasses35}, the classes are not all represented such that the class-compositions projections can be orthogonal. In other words, since we look at only a subspace of an ILR space, we are not looking at the full probability distribution. 

In the example of Figure \ref{fig:moreclasses35}, since $\tilde{p}_5$ is comparing 1 with the group of digits 7 and 8, the projection on this line of the class-compositions for 1 goes in an opposite direction than the one for the class-compositions for 7 and 8. The latter two are equal and half as long as the former. In this way, $\tilde{p}_5$ compares the probability for 1 with the group of probabilities for 7 and 8 with the same weight. In other words, in this subspace, the class-compositions for 7 and 8 are reweighted such that this group of two classes has the same weight as the group made of the single class 1.

Within this space, Shapley compositions can be explored as in the examples of Figures 
\ref{fig:3classesshapsum} and \ref{fig:4classesshapsum}, keeping in mind that this is a subspace of a full ILR space. 

\subsection{Angles, norms and projections}

An explanation can be summarised by sets of angles, norms and projections:
\begin{itemize}
\item The norm of a Shapley composition gives the strength of the contribution of the feature's value to the prediction. This measures the overall contribution of the feature, regardless of its direction.
\item The angle between two Shapley compositions informs about their orthogonality. Orthogonality suggests that the features are non-redundant for the given instance. A negative angle would suggest that the features have an opposite influence on the prediction.
\item The projections of a Shapley composition on the set of class-compositions inform in favour of, or against, which classes a feature's value is contributing.
\end{itemize}
To give a few examples, for the Iris example of Figure \ref{fig:3classesshapsum}, the norms for each Shapley composition are $\approx$ $1.27$, $1.02$, $0.36$ and $0.28$ respectively for the petal width, length and sepal width and length, confirming the features' importance one would expect from Figure \ref{fig:3classesshapsum}. The projection of the petal length's composition on the \emph{virginica} class-composition is $\approx 0.01$ confirming the low influence of this feature on the probability for this class. Finally, note that the cosine similarity between the Shapley compositions for petal length and width is close to one ($\approx 0.99$) which confirms these features are moving the distribution toward the same direction while the compositions for sepal and petal width have a cosine similarity of 0.45 confirming they point to complementary directions. 

\subsection{Histograms and parallel coordinates}

For a classification problem with at most 4 classes, an ILR space can be fully visualised within a single figure. However, for more classes
we cannot visualise the full ILR space and therefore have to explore subspaces. 
In this section we discuss alternative visualisations. 

The Shapley composition can be visualised using a bar plot like discrete probability distributions. Figure \ref{fig:histiris} shows the Shapley compositions of the Iris classification example as histograms. Note that in Figure \ref{fig:intro}, the histograms were showing the probability distributions as the successive perturbation of the base by the features' contribution. The histograms in this section refer to the visualisation of Shapley compositions for each feature separately. A more uniform histogram reflects less contribution of the feature's value to the change of the probability distribution (e.g. the sepal length in Figure \ref{fig:histiris}). In contrast, the Shapley compositions for the petal length and width have a high value for the \emph{versicolor} class, in comparison to the others. This confirms the contribution of these features toward this class.

\begin{figure}[ht]
  \centering
  \includegraphics[width=0.6\linewidth]{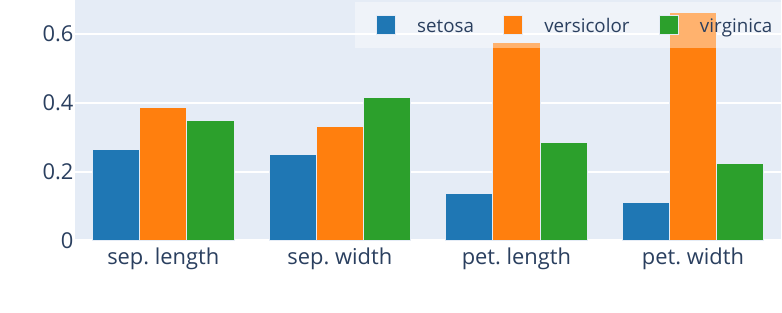}
  \caption{Shapley compositions visualised as histograms for the Iris classification example.}
  \label{fig:histiris}
\end{figure}
As another illustration, Figure \ref{fig:histmore} shows the Shapley compositions of the $10$-class digit recognition example. Here, and contrary to the visualisation of the compositions in an ILR space as in Section \ref{sec:balances}, one can analyse all parts of each composition within a single plot. In this example, the high value for digit 2 for the first principal component confirms the contribution of this feature toward this class.
\begin{figure}[ht]
  \centering
  \includegraphics[width=0.6\linewidth]{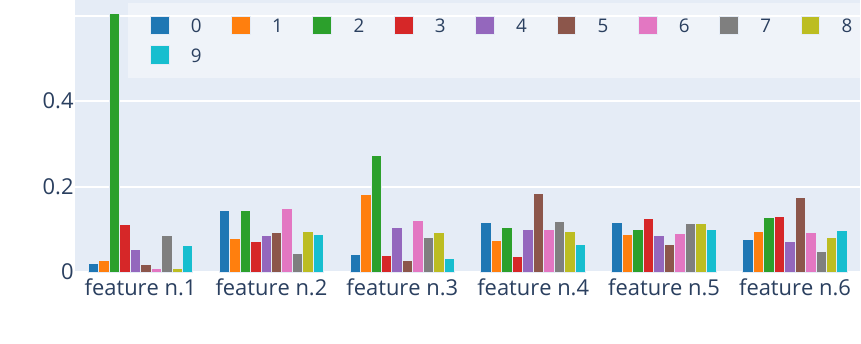}
  \vspace{-0.5cm}
  \caption{Shapley compositions visualised as histograms for the ten classes digit recognition example.}
  \label{fig:histmore}
\end{figure}

Another way to visualise the full compositions is with parallel coordinates. After sorting the features by their contribution (i.e.~the norm of their Shapley composition), the successive perturbation of the distribution can be visualised as probability lines from the base distribution to the prediction.
Figure \ref{fig:more:par} shows such a plot for the digit recognition example. With this visualisation, we can compactly see how the probability distribution is transformed by each feature contribution from the base distribution to the predicted one. In this example, the probability for digit 2 increases the most with the contribution of feature 1. This feature does not contribute in the change of the probability for digit 3 as suggested by the horizontal red segment. The next feature continues to increase the probability for digit 2 while decreasing the others.
\begin{figure}[ht]
  \centering
  \includegraphics[width=0.65\linewidth]{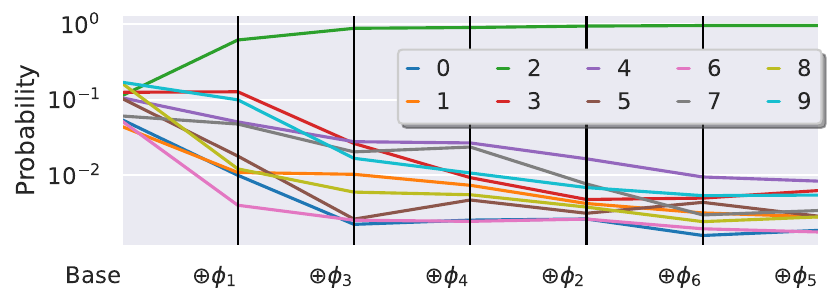}
  \caption{Parallel coordinates visualisation of the successive perturbation of the base distribution by Shapley compositions (ordered by importance, i.e., norm). The final distribution on the right side is the prediction.}
  \label{fig:more:par}
\end{figure}

\subsection{The estimation algorithm}

The estimation algorithm used in this work for computing the Shapley compositions is an adaptation of Algorithm 2 in \cite{vstrumbelj2014explaining}. Since the resulting Shapley compositions are approximations, the efficiency property does not necessarily hold. In order for the set of estimated Shapley compositions to respect the efficiency property, each Shapley composition is adjusted following a similar method as in the sampling approximation in the SHAP toolkit \cite{NIPS2017_7062}\footnote{\url{https://github.com/shap/shap/blob/master/shap/explainers/_sampling.py}}. We refer the reader to Appendix \ref{app:algo} and Appendix \ref{app:adjust} for more details.

Note that the Shapley composition framework can be applied on many different types of data, such as images. However, the main limitation is algorithmic: the complexity of the algorithm increases with the number of features, as with the standard Shapley value framework. Moreover, the estimation algorithm assumes that the features are independent. Estimating Shapley values without such assumption has been discussed in the literature \cite{AAS2021103502}. We leave the exploration of estimation algorithms of the Shapley compositions without the features-independence assumption and for data with a large number of features for future work.

\newpage
\section{Discussion and conclusion}
\label{sec:conclud}

The use of standard Shapley values for explaining multiclass machine learning models has been rarely discussed in the literature. However, the computation of the Shapley values on each output dimension one-by-one can be encountered. To be more precise, for an $D$-class problem ($D>2$), it may first sound natural to compute a Shapley value on the logit of the probability for each class resulting in a $D$-dimensional vector of the Shapley values. Even if the efficiency property holds with the standard addition, i.e.~the sum of the element-wise logit of the base distribution with such vectors for each feature is equal to the element-wise logit of the prediction, the path from the base to the prediction may go out of the simplex, i.e., the space of probability distributions, which is counter-intuitive and indeed incoherent. 
Moreover, such a strategy would require running $D$ independent explanations contrary to the Shapley composition approach which requires a single explanation process.

As far as we are aware, this paper is the first to propose an extension of the Shapley value framework to the multidimensional simplex for explaining a multiclass probabilistic prediction in machine learning. We saw how the formalisation of the standard Shapley value naturally extends to the simplex using the Aitchison geometry. The resulting Shapley quantity is a composition (distribution), i.e.~a vector living on the probability simplex. It is referred as \emph{Shapley composition} and explicates the contribution of a feature's value to a prediction. To be more precise, it tells how a feature's value moves the distribution from the base one to the predicted one on the simplex. We saw that the Shapley composition is the unique quantity that satisfies the linearity, symmetry and efficiency on the Aitchison simplex. 

The Aitchison geometry gives to the simplex an Euclidean vector space structure. For explaining a prediction, Shapley compositions can be visualised and analysed through angles, norms and projections. They inform on both the strength and the direction of each feature's value effect. Living on the probability simplex, i.e.~the same space as discrete probability distributions, the Shapley compositions can also be visualised as histograms. Parallel plots of probabilities can also be visualised to keep track of the change in the distribution induced by each feature's value.

The literature about the use of Shapley values in machine learning is extensive. Many estimation algorithms have been developed, many applications of the Shapley value have emerged, and large-scale experiments have been conducted. In contrast, our paper presents limited experimental results as simple proofs of concept and illustrations. However, the main contribution of this work is theoretical and methodological. We believe this work lays proper foundations to foster the research in explainable machine learning, especially for multidimensional and multiclass predictions.

\newpage
\begin{appendices}
\section{Proof of the uniqueness of Shapley compositions on the simplex}
\label{app:proof}

This section provides a proof of Theorem \ref{theo:shap}.

\subsection{Linearity, symmetry and efficiency}
Let's first show that the Shapley composition statisfies the \emph{linearity}, \emph{symmetry} and \emph{efficiency}.
\begin{proof}~
\paragraph{Linearity:}
  Let's consider the linear combination of predictions, or models, $\bm{h}(\bm{x}) = \alpha \odot \bm{f}(\bm{x}) \oplus \beta \odot \bm{g}(\bm{x})$.
  \begin{equation}
    \begin{aligned}
      \mathbb{E}^{\mathcal{A}}_{\text{Pr}}[\bm{h}(\bm{x})\mid \bm{x}_S] &= \ilr^{-1} \left( \mathbb{E}_{\text{Pr}} [\ilr \left(  \alpha \odot \bm{f}(\bm{x}) \oplus \beta \odot \bm{g}(\bm{x}) \right) \mid \bm{x}_S] \right),\\
                                                                        &= \ilr^{-1} \left( \mathbb{E}_{\text{Pr}} [\alpha \ilr \left( \bm{f}(\bm{x})\right) + \beta \ilr \left( \bm{g}(\bm{x}) \right) \mid \bm{x}_S] \right),\\
                                                                        &= \ilr^{-1} \left( \alpha \mathbb{E}_{\text{Pr}} [ \ilr \left( \bm{f}(\bm{x})\right) \mid \bm{x}_S] + \beta \mathbb{E}_{\text{Pr}} [ \ilr \left( \bm{g}(\bm{x}) \right) \mid \bm{x}_S] \right),\\
                                                                        &= \alpha \odot \ilr^{-1} \left( \mathbb{E}_{\text{Pr}} [ \ilr \left( \bm{f}(\bm{x})\right) \mid \bm{x}_S] \right) \oplus \beta \odot \ilr^{-1} \left( \mathbb{E}_{\text{Pr}} [ \ilr \left( \bm{g}(\bm{x}) \right) \mid \bm{x}_S] \right),\\
                                                                        &= \alpha \odot \mathbb{E}^{\mathcal{A}}_{\text{Pr}}[\bm{f}(\bm{x})\mid \bm{x}_S] \oplus \beta \odot \mathbb{E}^{\mathcal{A}}_{\text{Pr}}[\bm{g}(\bm{x})\mid \bm{x}_S].
    \end{aligned}
  \end{equation}
  Similarly, $\mathbb{E}^{\mathcal{A}}_{\text{Pr}}[\bm{h}(\bm{x})] = \alpha \odot \mathbb{E}^{\mathcal{A}}_{\text{Pr}}[\bm{f}(\bm{x})] \oplus \beta \odot \mathbb{E}^{\mathcal{A}}_{\text{Pr}}[\bm{g}(\bm{x})]$.\\

  Therefore, $\bm{v}_{\bm{h},\bm{x},\text{Pr}}(S) = \alpha \odot \bm{v}_{\bm{f},\bm{x},\text{Pr}}(S) \oplus \beta \odot \bm{v}_{\bm{g},\bm{x},\text{Pr}}(S)$, meaning that $\bm{v}$ is linear with respect to the model. The linearity of the contribution $\bm{c}$ naturally follows:
  \begin{equation}
    \begin{aligned}
      \forall i \in \mathcal{I},~\forall S\subseteq \mathcal{I}\backslash i,\\
      \bm{c}_{\bm{h},\bm{x},\text{Pr}}(i,S) &= \bm{v}_{\bm{h},\bm{x},\text{Pr}}({S\cup\{i\}}) \ominus \bm{v}_{\bm{h},\bm{x},\text{Pr}}(S),\\
      &= \left( \alpha \odot \bm{v}_{\bm{f},\bm{x},\text{Pr}}({S\cup\{i\}}) \oplus \beta \odot \bm{v}_{\bm{g},\bm{x},\text{Pr}}({S\cup\{i\}}) \right) \ominus \left( \alpha \odot \bm{v}_{\bm{f},\bm{x},\text{Pr}}(S) \oplus \beta \odot \bm{v}_{\bm{g},\bm{x},\text{Pr}}(S) \right),\\
      &= \alpha \odot \bm{v}_{\bm{f},\bm{x},\text{Pr}}({S\cup\{i\}}) \oplus \beta \odot \bm{v}_{\bm{g},\bm{x},\text{Pr}}({S\cup\{i\}}) \ominus\alpha \odot \bm{v}_{\bm{f},\bm{x},\text{Pr}}(S) \ominus \beta \odot \bm{v}_{\bm{g},\bm{x},\text{Pr}}(S),\\
      &= \alpha \odot \left( \bm{v}_{\bm{f},\bm{x},\text{Pr}}({S\cup\{i\}}) \ominus \bm{v}_{\bm{f},\bm{x},\text{Pr}}(S) \right) \oplus \beta \odot \left( \bm{v}_{\bm{g},\bm{x},\text{Pr}}({S\cup\{i\}}) \ominus \bm{v}_{\bm{g},\bm{x},\text{Pr}}(S)\right),\\
      &= \alpha \odot \bm{c}_{\bm{f},\bm{x},\text{Pr}}(i,S) \oplus \beta \odot \bm{c}_{\bm{g},\bm{x},\text{Pr}}(i,S).
    \end{aligned}
  \end{equation}
  And the linearity of the Shapley composition:
  \begin{equation}
    \begin{aligned}
    \forall i \in \mathcal{I},~~
      \bm{\phi}_i\left(\bm{h}\right) &= \frac{1}{d!}  \underset{\pi}{\bigoplus}\bm{c}_{\bm{h},\bm{x},\text{Pr}}(i,\pi^{<i}),\\
                            &= \frac{1}{d!}  \underset{\pi}{\bigoplus}\left( \alpha \odot \bm{c}_{\bm{f},\bm{x},\text{Pr}}(i,S) \oplus \beta \odot \bm{c}_{\bm{g},\bm{x},\text{Pr}}(i,S) \right),\\
                            &= \alpha \odot \left( \frac{1}{d!}  \underset{\pi}{\bigoplus} \bm{c}_{\bm{f},\bm{x},\text{Pr}}(i,S) \right) \oplus \beta \odot \left( \frac{1}{d!} \underset{\pi}{\bigoplus}    \bm{c}_{\bm{g},\bm{x},\text{Pr}}(i,S) \right),\\
                            &= \alpha \odot \bm{\phi}_i\left(\bm{f}\right) \oplus \beta\odot \bm{\phi}_i\left(\bm{g}\right).
    \end{aligned}
  \end{equation}
Thus, the Shapley composition is linear on the Aitchison simplex as a function of the prediction.

\paragraph{Symmetry} is straightforward.

\paragraph{Efficiency}
  \begin{equation}
    \begin{aligned}
      \underset{i=1}{\overset{d}\bigoplus} \bm{\phi}_i\left(\bm{f}\right) & = \underset{i=1}{\overset{d}\bigoplus}\left( \frac{1}{d!} \odot \underset{\pi}{\bigoplus} \bm{c}(i,\pi^{<i})\right),\\
                                                                 &=  \frac{1}{d!} \odot \underset{i=1}{\overset{d}\bigoplus}\left( \underset{\pi}{\bigoplus} \left( \bm{v}(\pi^{<i+1}) \ominus \bm{v}(\pi^{<i}) \right) \right),\\
                                                                 &=  \frac{1}{d!} \odot \underset{i=1}{\overset{d}\bigoplus}\left( \underbrace{\left(\underset{\pi}{\bigoplus} \bm{v}(\pi^{<i+1}) \right)}_{\bm{A}_{i+1}} \ominus \underbrace{\left( \underset{\pi}{\bigoplus} \bm{v}(\pi^{<i}) \right)}_{\bm{A}_{i}} \right),\\
                                                                 &=  \frac{1}{d!} \odot \underset{i=1}{\overset{d}\bigoplus}\left( \bm{A}_{i+1} \ominus \bm{A}_{i} \right),\\
                                                                 &=  \frac{1}{d!} \odot \left( \bm{A}_{d+1} \ominus \bm{A}_{1} \right), \text{~since we have a telescoping perturbation,}\\
                                                                 &=  \frac{1}{d!} \odot \left( \left( \underset{\pi}{\bigoplus} \bm{v}(\pi^{<d+1}) \right) \ominus \left( \underset{\pi}{\bigoplus} \bm{v}(\pi^{<1}) \right) \right),\\
                                                                 &=  \frac{1}{d!} \odot \left( \left( \underset{\pi}{\bigoplus} \bm{v}\left( \mathcal{I} \right) \right) \ominus \left( \underset{\pi}{\bigoplus} \bm{v}\left(\emptyset\right) \right) \right),\\
                                                                 &= \bm{v}\left( \mathcal{I} \right) \ominus \bm{v}\left(\emptyset \right), \text{~since $d!$ is the number of permutation,}\\
                                                                 &= \bm{f}(\bm{x}) \ominus \mathbb{E}^{\mathcal{A}}_{\text{Pr}}[\bm{f}(\bm{X})].
    \end{aligned}
  \end{equation}
Thus, the Shapley composition is \emph{linear}, \emph{symmetric} and \emph{efficient}.
\end{proof}
\subsection{Uniqueness}
Let's now show the uniqueness of the quantity satisfying the above three axiomatic properties on the simplex. The proof is highly inspired by the proofs for the standard Shapley value in game theory \cite{shapley1953value,myerson1997game}.
We first need the following Definition and Lemma.

Let's first consider the set $\mathcal{G}_d^{\mathcal{S}^D}$ of composition-valued set functions. This can be seen, in cooperative game theory, as the set of characteristic functions of games where the payoff, or worth, is a composition rather than a scalar. The function from Equation \ref{eq:valuefunctionsimplex} is such function.

For simplicity of the notation, let's consider one isometric-log-ratio space $\mathbb{R}^{D-1}$ isomorphic to the simplex $\mathcal{S}^D$. We therefore consider instead the set $\mathcal{G}_d^{\mathbb{R}^{D-1}}$ of vector-valued set functions\footnote{The tilde refers to the isometric-log-ratio transformation of a composition.}:
  \begin{equation}
    \mathcal{G}_d^{\mathbb{R}^{D-1}} = \{\tilde{\bm{v}}: 2^{\mathcal{I}}\to \mathbb{R}^{D-1} \mid \tilde{\bm{v}}(\emptyset) = \bm{0} \},  
  \end{equation}
  where $\mathcal{I}= \{1, 2, \dots d \}$. $\mathcal{G}_d^{\mathbb{R}^{D-1}}$ is $(2^d -1 )(D-1)$-dimensional, since $| 2^\mathcal{I} | = 2^d $, and the $-1$ corresponds to the empty set (with the constraint $\tilde{\bm{v}}(\emptyset) = \bm{0}$), and $D-1$ is the dimension of the functions' codomain.

  \begin{definition*} 
  Let $\{ \tilde{\bm{e}}^{(i)}\}_{1\leq i \leq D-1}$ be a linear basis of $\mathbb{R}^{D-1}$. The $i$th-vector-unanimity game $(\mathcal{I}, \tilde{\bm{v}}_T^{(i)})$, where $T\subseteq \mathcal{I} \backslash \emptyset$, is defined such that:
  \begin{equation}
    \label{eq:2}
    \forall S \subseteq \mathcal{I},~~~\tilde{\bm{v}}_T^{(i)}(S) = \left\lbrace
      \begin{array}{@{}ll@{}}
        \tilde{\bm{e}}^{(i)}, & \text{if}\ T\subseteq S \\
        \bm{0}, & \text{otherwise.}
      \end{array}\right.
  \end{equation}
  \end{definition*}

  \begin{lemma*} The set $\{ \tilde{\bm{v}}_T^{(i)}\in \mathcal{G}_d^{\mathbb{R}^{D-1}} \mid T\subseteq \mathcal{I} \backslash \emptyset , 1 \leq i \leq D-1 \}$ forms a linear basis of $\mathcal{G}_d^{\mathbb{R}^{D-1}}$.
  \end{lemma*}

  \begin{proof}
    There are $(2^d -1 ) \times (D-1)$ $i$th-vector-unanimity games, the same number as the dimensionality of $\mathcal{G}_d^{\mathbb{R}^{D-1}}$. We therefore just need to prove that they are linearly independent towards a contradiction.

    Let's assume that:

    \begin{equation}
      \sum _{\substack{T \subseteq \mathcal{I} \backslash \emptyset \\ 1 \leq i \leq D-1}} \alpha_T^{(i)} \tilde{\bm{v}}_T^{(i)} = \bm{0},~\text{and}~\exists~\alpha_T^{(i)} \neq 0.
    \end{equation}
    Let $T_0$ be a set of minimal size in $\{ T\subseteq \mathcal{I} \backslash \emptyset \mid \exists~\alpha_T^{(i)} \neq 0 \}$.

    Then,
    \begin{equation}
      \sum _{\substack{T \subseteq \mathcal{I} \backslash \emptyset \\ 1 \leq i \leq D-1}} \alpha_T^{(i)} \tilde{\bm{v}}_T^{(i)}(T_0) = \sum _{\substack{T \subseteq T_0 \\ T \neq \emptyset \\ 1 \leq i \leq D-1}} \alpha_T^{(i)} \tilde{\bm{v}}_T^{(i)}(T_0),
    \end{equation}
    because, by definition of the $i$th-vector-unanimity game, for all $T \subseteq \mathcal{I} \backslash \emptyset$ not in $T_0$, $\tilde{\bm{v}}_T^{(i)}(T_0) = \bm{0}$.

    And because, $T_0$ is the minimal set in $\{ T\subseteq \mathcal{I} \backslash \emptyset \}$ such that $\exists~\alpha_{T_0}^{(i)} \neq 0 $, $\forall T\subset T_0$ and $1 \leq i \leq D-1$, $\alpha_T^{(i)} = 0$.

    Then,
    \begin{equation}
      \begin{aligned}
        \sum _{\substack{T \subseteq \mathcal{I} \backslash \emptyset \\ 1 \leq i \leq D-1}} \alpha_T^{(i)} \tilde{\bm{v}}_T^{(i)}(T_0) &= \sum _{1 \leq i \leq D-1} \alpha_{T_0}^{(i)} \tilde{\bm{v}}_{T_0}^{(i)}(T_0),\\
                                                                                                                                        &=\sum _{1 \leq i \leq D-1} \alpha_{T_0}^{(i)} \tilde{\bm{e}}^{(i)} \neq \bm{0},
      \end{aligned}
    \end{equation}
    because $\exists~\alpha_{T_0}^{(i)} \neq 0$ and $\{ \tilde{\bm{e}}^{(i)}\}_{1\leq i \leq D-1}$ are linearly independent.

    This is indeed a contradiction. Therefore, the elements of $\{ \tilde{\bm{v}}_T^{(i)}\in \mathcal{G}_d^{\mathbb{R}^{D-1}} \mid T\subseteq \mathcal{I} \backslash \emptyset , 1 \leq i \leq D-1 \}$ are linearly independent and form a linear basis of $\mathcal{G}_d^{\mathbb{R}^{D-1}}$.
  \end{proof}

\begin{corollary*}
The set $\{ {\bm{v}}_T^{(i)} = \ilr^{-1} \left(\tilde{\bm{v}}_T^{(i)} \right)\in \mathcal{G}_d^{\mathcal{S}^{D}} \mid T\subseteq \mathcal{I} \backslash \emptyset , 1 \leq i \leq D-1 \}$ forms an Aitchison linear basis of $\mathcal{G}_d^{\mathcal{S}^{D}}$,
\end{corollary*}
where $\bm{v}_T^{(i)}$ can be referred as a $i$th-composition-unanimity game.

Thus, any composition-valued set function $\bm{v}$ can be written with a unique set of alphas as:
  \begin{equation}
    \forall S \subseteq \mathcal{I},~ \bm{v}(S) = \bigoplus_{\substack{T \subseteq \mathcal{I} \backslash \emptyset \\ 1 \leq i \leq D-1}} \alpha_T^{(i)} \odot \bm{v}_T^{(i)}(S).
  \end{equation}
  
Let's now prove that for any composition-valued set function $\bm{v} \in \mathcal{G}_d^{\mathcal{S}^D}$, there is an unique set $\{\bm{\phi}_k\}_{1 \leq k \leq d}$ of composition-valued function that satisfies the three axiomatic properties.
\begin{proof}
Let's first show there is an unique set $\{\bm{\phi}_k\}_{1 \leq k \leq d}$ of composition-valued function that satisfies the three axiomatic properties for a powered $i$th-composition-unanimity game $\beta \odot \bm{v}_{T}^{(i)}$ where $\beta \in \mathbb{R}$.

  With the \emph{linearity} and the \emph{efficiency} axiomatic properties, we have\footnote{In the core of the paper, the Shapley compositions are written as functions of a model. Here, without loss of generality, they are written as functions of a characteristic function.}:
  \begin{equation}
    \bigoplus_{k \in T} \bm{\phi}_k \left( \beta \odot \bm{v}_T^{(i)}\right) = \beta \odot \bm{e}^{(i)},~\text{and}~\forall j \notin T,~\bm{\phi}_j\left(\beta \odot v_T^{(i)}\right)=\bm{u},
  \end{equation}
  where $\{ \bm{e}^{(i)} = \ilr^{-1} \left( \tilde{\bm{e}}^{(i)} \right)\}_{1\leq i \leq D-1}$ forms an Aitchison linear basis of the simplex, and $\bm{u}\in\mathcal{S}^D$ is the uniform distribution, i.e. the neutral element for the perturbation: $\bm{u} = \ilr^{-1} \left( \bm{0} \right)$.

  Let $(k,l) \in T^2$, $\forall S \subseteq \mathcal{I} \backslash \{k,l\}$, we have $\beta \odot \bm{v}_T^{(i)} \left( S\cup \{k\} \right) = \beta \odot \bm{v}_T^{(i)} \left( S\cup \{l\} \right)$. Therefore, due to the \emph{symmetry}, $\bm{\phi}_k \left( \beta \odot \bm{v}_T^{(i)} \right) = \bm{\phi}_l \left( \beta \odot \bm{v}_T^{(i)} \right)$ such that $\forall k \in T$ the $\bm{\phi}_k$ are equal.
  
  Thus,
  \begin{equation}
  \forall k \in \mathcal{I},~~~
    \bm{\phi}_k\left( \beta \odot \bm{v}_T^{(i)} \right) = \left\lbrace
      \begin{array}{@{}ll@{}}
        \frac{\beta}{|T|} \odot \bm{e}^{(i)}, & \text{if}\ k \in T \\
        \bm{u}, & \text{otherwise.}
      \end{array}\right.
  \end{equation}

  Therefore, the set $\{ \bm{\phi}_k\}_{1 \leq k \leq d}$ of composition-valued function, respecting the axiomatic properties, is uniquely defined for a powered $i$th-composition-unanimity game.

  Finally, there is a unique set $\{ \bm{\phi}_k\}_{1 \leq k \leq d}$ that satisfies the three axiomatic properties for any composition-valued set function since such function is uniquely represented by a linear combination of $i$th-composition-unanimity games and the functions in $\{ \bm{\phi}_k\}_{1 \leq k \leq d}$ are linear with respect to the characteristic function.
  \end{proof}

\newpage
\section{Class-compositions}
\label{app:classcompo}
A $k$-class-composition $\bm{c}^{(k)}  \in \mathcal{S}^D$ is defined as an unit norm composition going straight to the direction of the $k$th class. This is a discrete probability distribution with maximum probability for the $k$th class and uniform values for the others. The $i$th part of $\bm{c}^{(k)}$ is:
\begin{equation}
    c_i^{(k)} = \left\lbrace
  \begin{array}{@{}ll@{}}
    1-(D-1)p, & \text{if}\ i=k \\
    p, & \text{otherwise,}
  \end{array}\right.
\end{equation}
where $0<p<\frac{1}{D}$. We want the Aitchison norm of each class-composition to be one:
\begin{equation}
  \begin{aligned}
    \forall k \in \{1, \dots D \},~~~~~~\lVert \bm{c}^{(k)} \rVert_{\mathcal{A}} = 1 \iff& \sqrt{\frac{1}{2D} \sum_{i=1}^D \sum_{j=1}^D \left( \log \frac{c_i^{(k)}}{c_j^{(k)}} \right)^2} = 1,\\
    \text{for clarity,}&\\\text{we drop the superscript $(k)$ from the equations,}&\\
    \iff& \sqrt{\frac{1}{2D} \sum_{i=1}^{D}\left( \left(D-1 \right) \left( \log \frac{c_i}{p} \right)^2 + \left( \log \frac{c_i}{1-(D-1)p} \right)^2 \right)} = 1,\\
    \iff& \sqrt{\frac{1}{2D} 2(D-1) \left( \log \frac{p}{1-(D-1)p} \right)^2} =1,\\
    \text{since $p<\frac{1}{D}$ and the norm is positive:}&\\
    \iff& \sqrt{\frac{D-1}{D}}\log \frac{1-(D-1)p}{p} =1,\\
    \iff& p = \frac{\exp \left( -\sqrt{\frac{D}{D-1}} \right)}{1+ (D-1)\exp \left( -\sqrt{\frac{D}{D-1}} \right )}.
  \end{aligned}
\end{equation}
To summarise, the $i$th part of a $k$-class-composition $\bm{c}^{(k)} \in \mathcal{S}^D$ is given by:
\begin{equation}
    c_i^{(k)} =
  \frac{1}{1+ (D-1)\exp \left( -\sqrt{\frac{D}{D-1}} \right )} \left( \left\lbrace \begin{array}{@{}ll@{}}
     1, & \text{if}\ i=k \\
     \exp \left( -\sqrt{\frac{D}{D-1}} \right), & \text{otherwise,}
  \end{array}\right.\right).
\end{equation}
In this way, $\bm{c}^{(k)}$ is going straight to the direction of class $k$ and uniformly against all the others with a unit norm.

\newpage
\section{Estimation of the Shapley compositions}
\label{app:algo}
This section presents the estimation algorithm we used to estimate the Shapley compositions in our experiments. The algorithm is an adaptation of Algorithm 2 in \cite{vstrumbelj2014explaining}. 

Let $d$ be the number of features. We want to optimally distribute $m_{\text{max}}$ drawn samples over the $d$ features. Let $\hat{\bm{\phi}_i}$ be the estimation of the Shapley composition for the $i$th feature. We want to minimise the sum of squared errors: $\sum_{i=1}^{d} \| \hat{\bm{\phi}_i} \ominus \bm{\phi}_i \|_{\mathcal{A}}^2$.

Since $\hat{\bm{\phi}_i}$ is a (Aitchison) sample mean we have: $\tilde{\hat{\bm{\phi}_i}} \approx \mathcal{N}\left(\tilde{\bm{\phi}}_i, \frac{1}{m_i}\bm{\Sigma}^{(i)}\right)$ and $\tilde{\hat{\bm{\phi}_i}} - \tilde{\bm{\phi}}_i \approx \mathcal{N}\left(\bm{0}, \frac{1}{m_i}\bm{\Sigma}^{(i)}\right)$ where the tilde refers to the ILR transformation.
Let $\bm{\Delta_i} = \tilde{\hat{\bm{\phi}_i}} - \tilde{\bm{\phi}}_i$ and $Z_i = \| \hat{\bm{\phi}_i} \ominus \bm{\phi}_i \|_{\mathcal{A}} = \| \tilde{\hat{\bm{\phi}_i}} - \tilde{\bm{\phi}}_i \|_2 = \| \bm{\Delta}_i \|_2$. The expectation of the sum of squared errors is:
\begin{equation}
  \begin{aligned}
    \mathbb{E}\left[ \sum_{i=1}^{d} Z_i^2 \right] &=   \sum_{i=1}^{d} \mathbb{E}\left[ Z_i^2 \right],\\
                                                  & =\sum_{i=1}^{d} \mathbb{E}\left[ \sum_{j=1}^{D-1}\Delta_{ij}^2 \right],\\
                                                  & =\sum_{i=1}^{d} \sum_{j=1}^{D-1} \mathbb{E}\left[\Delta_{ij}^2 \right],\\
                                                  & =\sum_{i=1}^{d} \sum_{j=1}^{D-1} \frac{1}{m_i}\Sigma_{jj}^{(i)}, \text{ since } \Delta_{ij}\approx \mathcal{N}\left(0, \frac{1}{m_i}\Sigma_{jj}^{(i)}\right),\\
    &= \sum_{i=1}^d \frac{1}{m_i} \tr \bm{\Sigma}^{(i)}.
  \end{aligned}
\end{equation}

When a sample is drawn, the feature for which the sample will be used for improving the Shapley composition estimation is chosen to maximise $\frac{\tr \bm{\Sigma}^{(i)}}{m_i} - \frac{\tr \bm{\Sigma}^{(i)}}{m_i+1}$. Like in \cite{vstrumbelj2014explaining}, this is summarised in Algorithm \ref{alg:2}.
\begin{algorithm}
   \caption{Adaptation of the Algorithm 1 from \cite{vstrumbelj2014explaining} for approximating the Shapley composition of the $i$th feature, with model $\bm{f}$, instance $\bm{x}\in\mathcal{X}$ and $m$ drawn samples.}
   \label{alg:1}
\begin{algorithmic}
   \STATE Initialise $\bm{\phi_i}\leftarrow \ilr^{-1}(\bm{0})$
   \FOR{$1$ {\bfseries to} $m$}
   \STATE Randomly select a permutation $\pi$ of the set of indexes $\mathcal{I}$,
   \STATE Randomly select a sample $\bm{w}\in\mathcal{X}$,
   \STATE Construct two instances:
   \begin{itemize}
     \item $\bm{b}_1$: which takes the values from $\bm{x}$ for the $i$th feature and the features indexed before $i$ in the order given by $\pi$, and takes the values from $\bm{w}$ otherwise,
     \item $\bm{b}_2$: which takes the values from $\bm{x}$ for the features indexed before $i$ in the order given by $\pi$, and takes the values from $\bm{w}$ otherwise.
     \end{itemize}
   \STATE $\bm{\phi}_i \leftarrow \bm{\phi}_i \oplus \bm{f}(\bm{b}_1) \ominus \bm{f}(\bm{b}_2) $
   \ENDFOR
   \STATE $\bm{\phi}_i \leftarrow \frac{1}{m}\odot \bm{\phi}_i$
\end{algorithmic}
\end{algorithm}

\begin{algorithm}
   \caption{Adaptation of the Algorithm 2 from \cite{vstrumbelj2014explaining} for approximating all the Shapley compositions by optimally distributing a maximum number of samples $m_{\text{max}}$ over the $d$ features, with model $\bm{f}$, instance $\bm{x}\in\mathcal{X}$ and $m_{\text{min}}$ the minimum number of samples for each feature estimation.}
   \label{alg:2}
   \begin{algorithmic}
     \STATE Initialisation: $m_{i} \leftarrow 0$, $\tilde{\bm{\phi}}_i \leftarrow \bm{0}$, $\forall i \in \{1, \dots d\}$,
     \WHILE{$\displaystyle \sum_{i=1}^d m_i < m_{\text{max}}$}
     \IF{$\forall i, m_i \geq m_{\text{min}}$}
     \STATE $j = \underset{i}{\text{argmax}} \left( \frac{\tr \bm{\Sigma}^{(i)}}{m_i} - \frac{\tr \bm{\Sigma}^{(i)}}{m_i+1} \right)$,
     \ELSE
     \STATE pick a $j$ such that $m_j < m_{\text{min}}$,
     \ENDIF
     \STATE $\tilde{\bm{\phi}}_j \leftarrow \tilde{\bm{\phi}}_j$ $+$ isometric-log-ratio transformed result of Algorithm \ref{alg:1} for the $j$th feature and $m=1$,
     \STATE update $\bm{\Sigma}^{(j)}$ using an incremental algorithm,
     \STATE $m_j \leftarrow m_j+1$
     \ENDWHILE
     \STATE $\tilde{\bm{\phi}}_i \leftarrow \frac{\tilde{\bm{\phi}}_i}{m_i}$, $\forall i \in \{1, \dots d\}$
     \STATE $\bm{\phi}_i \leftarrow \ilr^{-1} \left( \tilde{\bm{\phi}}_i\right)$
\end{algorithmic}
\end{algorithm}

\newpage
\section{Adjustement of the estimated Shapley compositions for efficiency}
\label{app:adjust}

In practice, the computation of the Shapley values has an exponential time complexity and we do not have necessarily access to the true distribution of the data. The Shapley values are therefore approximated using estimation algorithms like for instance the one presented in the previous section. However, since the obtained values are approximations, they do not necessarily respect the desired efficiency property. This point is often overlooked in the literature. In this section, we write down an adjustment strategy of the estimated Shapley compositions for them to respect the efficiency property. This is a similar strategy as in the sampling approximation of the Shapley values in the SHAP toolkit\footnote{\url{https://github.com/shap/shap/blob/master/shap/explainers/_sampling.py}}.

Let $\{\hat{\bm{\phi}}_i\}_{1\leq i \leq d}$ be the estimated Shapley compositions (given by the Algorithm \ref{alg:2} in our experiments). Let $\displaystyle \bm{s}_{err} = \bm{f}(\bm{x}) \ominus \bm{f}_0 \ominus \underset{i=1}{\overset{d}\bigoplus} \hat{\bm{\phi}}_i$, where $\bm{f}_0$ is the base distribution, be the error composition on the pertubation of all Shapley compositions, i.e.~the error that makes the efficiency property unfulfilled. In order to respect the efficiency property, we want this error to be the neutral element of the perturbation, i.e.~the ``zero'' in the sense of the Aitchison geometry: the uniform distribution. We could simply perturb each estimated Shapley compositions by $\frac{1}{d}\odot\bm{s}_{err}$ however this would move each Shapley composition by the same amount while we want to allow the Shapley compositions with a higher estimation variance (i.e.~with a precision likely to be lower) to move more than the Shapley compositions with a smaller estimation variance (i.e.~with a precision likely to be higher).

The $i$th adjustment is therefore weighted by a scalar $w_i = w\left(\tr\left(\bm{\Sigma}^{(i)}\right)\right)$, where $w$ is an increasing function, and where $\displaystyle \sum_{i=1}^d w_i= 1$. Similarly to the SHAP toolkit implementation, we choose $w$ as:
\begin{equation}
  w_i = w\left(\tr\left(\bm{\Sigma}^{(i)}\right)\right) = \frac{v_i}{\displaystyle 1+\sum_{j=1}^{d}v_j},\text{ where } v_i = \frac{\tr\left(\bm{\Sigma}^{(i)}\right)}{\displaystyle \epsilon \max_j \tr\left(\bm{\Sigma}^{(j)}\right)}.
\end{equation}
The $i$th estimated Shapley composition is then asjusted as follow:
\begin{equation}
  \hat{\bm{\phi}}_i \leftarrow \hat{\bm{\phi}}_i \oplus \left( w_i \odot \bm{s}_{err}\right).
\end{equation}
In this way, when $\epsilon$ goes to zero\footnote{In our experiments, $\epsilon = 10^{-6}$.}, the efficiency property is respected for the adjusted Shapley compositions and more weight is given to the adjustments of the Shapley compositions with a higher estimation variance.

\end{appendices}
\newpage
\section*{Acknowledgments}
The work of PF and MPN was supported by TAILOR\footnote{\url{https://tailor-network.eu}}, a European research network funded by the EU Horizon 2020 research and innovation programme under GA No 952215. This work wouldn't have happened without a research visit of PGN at the University of Bristol made possible by the TAILOR Connectivity Fund. The work of PGN and JFB was also supported by the LIAvignon chair.

We thank Telmo de Menezes e Silva Filho from the University of Bristol for suggesting parallel coordinates to visualise Shapley compositions. We also thank the anonymous reviewers for helpful comments.

\bibliography{biblio}

@incollection{NIPS2017_7062,
title = {A Unified Approach to Interpreting Model Predictions},
author = {Lundberg, Scott M and Lee, Su-In},
booktitle = {Advances in Neural Information Processing Systems 30},
editor = {I. Guyon and U. V. Luxburg and S. Bengio and H. Wallach and R. Fergus and S. Vishwanathan and R. Garnett},
pages = {4765--4774},
year = {2017},
publisher = {Curran Associates, Inc.}
}

@article{aitchison1980,
 author = {John Aitchison and S. M. Shen},
 journal = {Biometrika},
 number = {2},
 pages = {261--272},
 publisher = {[Oxford University Press, Biometrika Trust]},
 title = {Logistic-Normal Distributions: Some Properties and Uses},
 volume = {67},
 year = {1980}
}

@book{myerson1997game,
  title={Game theory},
  author={Myerson, Roger B},
  year={1997},
  publisher={Harvard university press}
}

@article{vstrumbelj2014explaining,
  title={Explaining prediction models and individual predictions with feature contributions},
  author={{\v{S}}trumbelj, Erik and Kononenko, Igor},
  journal={Knowledge and information systems},
  volume={41},
  pages={647--665},
  year={2014},
  publisher={Springer}
}

@INPROCEEDINGS{datta2016,
  author={Datta, Anupam and Sen, Shayak and Zick, Yair},
  booktitle={2016 IEEE Symposium on Security and Privacy (SP)}, 
  title={Algorithmic Transparency via Quantitative Input Influence: Theory and Experiments with Learning Systems}, 
  year={2016},
  volume={},
  number={},
  pages={598-617}
  }

@INPROCEEDINGS{shapley1953value,
  title={A value for n-person games},
  booktitle={Contributions to the Theory of Games II},
  author={Shapley, Lloyd S and others},
  year={1953},
  publisher={Princeton University Press Princeton},
  pages={307-317}
}

@book{pawlowskymodeling,
  title={Modeling and Analysis of Compositional Data},
  author={Pawlowsky-Glahn, Vera and Egozcue, Juan José and Tolosana-Delgado, Raimon},
  publisher={John Wiley \& Sons},
  year={2015}
}

@article{aitchison1982,
 author = {John Aitchison},
 journal = {Journal of the Royal Statistical Society. Series B (Methodological)},
 number = {2},
 pages = {139--177},
 publisher = {[Royal Statistical Society, Wiley]},
 title = {The Statistical Analysis of Compositional Data},
 volume = {44},
 year = {1982}
}

@incollection{aitchison2001,
    author = {John Aitchison},
    title = {Simplicial inference},
    booktitle =     {Algebraic Methods in Statistics and Probability},
    editor =    {Marlos A. G. Viana, Donald St. P. Richards},
    publisher = {American Mathematical Society},
    year =      {2001},
    series =    {Contemporary Mathematics 287}
}

@article{egozcue2003isometric,
  title={Isometric logratio transformations for compositional data analysis},
  author={Egozcue, Juan Jos{\'e} and Pawlowsky-Glahn, Vera and Mateu-Figueras, Gl{\`o}ria and Barcelo-Vidal, Carles},
  journal={Mathematical geology},
  volume={35},
  number={3},
  pages={279--300},
  year={2003},
  publisher={Springer}
}

@article{egozcue2005groups,
  title={Groups of parts and their balances in compositional data analysis},
  author={Egozcue, Juan José and Pawlowsky-Glahn, Vera},
  journal={Mathematical Geology},
  volume={37},
  number={7},
  pages={795--828},
  year={2005},
  publisher={Springer}
}

@article{pedregosa2011scikit,
  title={Scikit-learn: Machine learning in Python},
  author={Pedregosa, Fabian and Varoquaux, Ga{\"e}l and Gramfort, Alexandre and Michel, Vincent and Thirion, Bertrand and Grisel, Olivier and Blondel, Mathieu and Prettenhofer, Peter and Weiss, Ron and Dubourg, Vincent and others},
  journal={the Journal of machine Learning research},
  volume={12},
  pages={2825--2830},
  year={2011},
  publisher={JMLR. org}
}

@article{egozcue2011evidence,
author = {Egozcue, Juan José and Pawlowsky-Glahn, Vera},
year = {2011},
month = {05},
pages = {},
title = {Evidence information in Bayesian updating},
journal = {Proc. International Workshop on Compositional Data Analysis}
}

@article{egozcue2018evidence,
    title={Evidence functions: a compositional approach to information},
    volume={1},
    number={2},
    journal={SORT-Statistics and Operations Research Transactions},
    author={Egozcue, Juan José and Vera, Pawlowsky-Glahn},
    year={2018},
    month={Dec.},
    pages={101-124}
}

@phdthesis{noe2023representing,
  title={Representing evidence for attribute privacy: bayesian updating, compositional evidence and calibration},
  author={No{\'e}, Paul-Gauthier},
  year={2023},
  school={Universit{\'e} d'Avignon}
}

@article{angelov2021explainable,
  title={Explainable artificial intelligence: an analytical review},
  author={Angelov, Plamen P and Soares, Eduardo A and Jiang, Richard and Arnold, Nicholas I and Atkinson, Peter M},
  journal={Wiley Interdisciplinary Reviews: Data Mining and Knowledge Discovery},
  volume={11},
  number={5},
  pages={e1424},
  year={2021},
  publisher={Wiley Online Library}
}

@article{wu2003probability,
  title={Probability estimates for multi-class classification by pairwise coupling},
  author={Wu, Ting-Fan and Lin, Chih-Jen and Weng, Ruby},
  journal={Advances in Neural Information Processing Systems},
  volume={16},
  year={2003}
}

@article{AAS2021103502,
title = {Explaining individual predictions when features are dependent: More accurate approximations to Shapley values},
journal = {Artificial Intelligence},
volume = {298},
pages = {103502},
year = {2021},
author = {Kjersti Aas and Martin Jullum and Anders Løland}
}

@InProceedings{Utkin2023,
author="Utkin, Lev V.
and Petrov, Artem
and Konstantinov, Andrei",
editor="Arseniev, Dmitry G.
and Aouf, Nabil",
title="Modifications of SHAP for Local Explanation of Function-Valued Predictions Using the Divergence Measures",
booktitle="Cyber-Physical Systems and Control II",
year="2023",
publisher="Springer International Publishing",
address="Cham",
pages="52--64"}

@Article{lamens_explaining_2023,
AUTHOR = {Lamens, Alec and Bajorath, Jürgen},
TITLE = {Explaining Multiclass Compound Activity Predictions Using Counterfactuals and Shapley Values},
JOURNAL = {Molecules},
VOLUME = {28},
YEAR = {2023},
NUMBER = {14},
ARTICLE-NUMBER = {5601}
}

@article{utkin_imprecise_2021,
	title = {An Imprecise {SHAP} as a Tool for Explaining the Class Probability Distributions under Limited Training Data},
	rights = {{arXiv}.org perpetual, non-exclusive license},
journal={arXiv preprint arXiv:2106.09111},
	author = {Utkin, Lev V. and Konstantinov, Andrei V. and Vishniakov, Kirill A.},
	urldate = {2024-01-31},
	year = {2021},
}

@inproceedings{ribeiro_why_2016,
author = {Ribeiro, Marco Tulio and Singh, Sameer and Guestrin, Carlos},
title = {"Why Should I Trust You?": Explaining the Predictions of Any Classifier},
year = {2016},
publisher = {Association for Computing Machinery},
address = {New York, NY, USA},
pages = {1135–1144},
numpages = {10},
location = {San Francisco, California, USA},
series = {KDD '16}
}

@article{apley_visualizing_2020,
    author = {Apley, Daniel W. and Zhu, Jingyu},
    title = "{Visualizing the Effects of Predictor Variables in Black Box Supervised Learning Models}",
    journal = {Journal of the Royal Statistical Society Series B: Statistical Methodology},
    volume = {82},
    number = {4},
    pages = {1059-1086},
    year = {2020},
    month = {06}}

@article{greenwell_simple_2018,
	title = {A simple and effective model-based variable importance measure},
	journal = {{arXiv} preprint {arXiv}:1805.04755},
	author = {Greenwell, Brandon M and Boehmke, Bradley C and {McCarthy}, Andrew J},
	year = {2018},
}

@article{goldstein_peeking_2015,
	title = {Peeking inside the black box: Visualizing statistical learning with plots of individual conditional expectation},
	volume = {24},
	pages = {44--65},
	number = {1},
	journal = {Journal of Computational and Graphical Statistics},
	author = {Goldstein, Alex and Kapelner, Adam and Bleich, Justin and Pitkin, Emil},
	year = {2015},
	note = {Publisher: Taylor \& Francis},
}

@article{fisher_all_2019,
	title = {All Models are Wrong, but Many are Useful: Learning a Variable's Importance by Studying an Entire Class of Prediction Models Simultaneously},
	volume = {20},
	pages = {1--81},
	number = {177},
	journal = {J. Mach. Learn. Res.},
	author = {Fisher, Aaron and Rudin, Cynthia and Dominici, Francesca},
	year = {2019},
}

@article{breiman_random_2001,
	title = {Random forests},
	volume = {45},
	pages = {5--32},
	number = {1},
	journal = {Machine learning},
	author = {Breiman, Leo},
	date = {2001},
	note = {Publisher: Springer},
    year = {2001}}

@article{franceschi2024explaining,
  title={Explaining Probabilistic Models with Distributional Values},
  author={Franceschi, Luca and Donini, Michele and Archambeau, C{\'e}dric and Seeger, Matthias},
  journal={arXiv preprint arXiv:2402.09947},
  year={2024}
}

@InProceedings{franceschi2023,
author="Franceschi, Luca
and Zor, Cemre
and Zafar, Muhammad Bilal
and Detommaso, Gianluca
and Archambeau, Cedric
and Madl, Tamas
and Donini, Michele
and Seeger, Matthias",
editor="Chen, Hao
and Luo, Luyang",
title="Explaining Multiclass Classifiers with Categorical Values: A Case Study in Radiography",
booktitle="Trustworthy Machine Learning  for Healthcare",
year="2023",
publisher="Springer Nature Switzerland",
address="Cham",
pages="11--24"}

@article{yoo2020,
    author = {Yoo, Tae Keun and Ryu, Ik Hee and Choi, Hannuy and Kim, Jin Kuk and Lee, In Sik and Kim, Jung Sub and Lee, Geunyoung and Rim, Tyler Hyungtaek},
    title = "{Explainable Machine Learning Approach as a Tool to Understand Factors Used to Select the Refractive Surgery Technique on the Expert Level}",
    journal = {Translational Vision Science \& Technology},
    volume = {9},
    number = {2},
    pages = {8-8},
    year = {2020},
    month = {02}
    }

@article{gunning_xaiexplainable_2019,
author = {David Gunning  and Mark Stefik  and Jaesik Choi  and Timothy Miller  and Simone Stumpf  and Guang-Zhong Yang },
title = {XAI—Explainable artificial intelligence},
journal = {Science Robotics},
volume = {4},
number = {37},
pages = {eaay7120},
year = {2019},
doi = {10.1126/scirobotics.aay7120},
}

@inproceedings{sokol_explainability_2020,
author = {Sokol, Kacper and Flach, Peter},
title = {Explainability fact sheets: a framework for systematic assessment of explainable approaches},
year = {2020},
publisher = {Association for Computing Machinery},
address = {New York, NY, USA},
booktitle = {Proceedings of the 2020 Conference on Fairness, Accountability, and Transparency},
pages = {56–67},
numpages = {12},
location = {Barcelona, Spain},
series = {FAT* '20}
}

\end{document}